\documentclass[format=sigconf,nonacm]{acmart}
\usepackage{comment}
\usepackage[utf8]{inputenc}
\usepackage{amsmath}
\usepackage{comment}
\usepackage{mathrsfs}
\usepackage{amssymb}
\usepackage{amsthm}
\usepackage{wrapfig}
\usepackage{pgfplots}

\newtheorem{theorem}{Theorem}[section]
\newtheorem{definition}[theorem]{Definition}
\newtheorem{lemma}[theorem]{Lemma}
\newtheorem{example}[theorem]{Example}
\newtheorem{assumption}[theorem]{Assumption}
\newtheorem{proposition}[theorem]{Proposition}

\title{Dynamic Modeling and Equilibria in Fair Decision Making}
\author[1]{Joshua Williams}
\affiliation{Carnegie Mellon University \and
     jnwillia@cs.cmu.edu}
     
\author[1]{J. Zico Kolter}
\affiliation{Carnegie Mellon University \\
     and Bosch Center for AI \and
     zkolter@cs.cmu.edu}

\begin{abstract}

Recent studies on fairness in automated decision making systems have both investigated the potential \emph{future impact} of these decisions on the population at large, and emphasized that imposing ``typical'' fairness constraints such as demographic parity or equality of opportunity does not guarantee a benefit to disadvantaged groups. However, these previous studies have focused on either simple one-step cost/benefit criteria, or on discrete underlying state spaces. In this work, we first propose a natural \emph{continuous} representation of population state, governed by the Beta distribution, using a loan granting setting as a running example. Next, we apply a model of population \emph{dynamics} under lending decisions, and show that when conditional payback probabilities are estimated correctly 1) ``optimal'' behavior by lenders can lead to ``Matthew Effect'' bifurcations (i.e., ``the rich get richer and the poor get poorer''), but that 2) many common fairness constraints on the allowable policies cause groups to converge to the same equilibrium point. Last, we contrast our results in the case of misspecified conditional probability estimates with prior work, and show that for this model, different levels of group misestimation guarantees that even fair policies lead to bifurcations. We illustrate some of the modeling conclusions on real data from credit scoring.

\end{abstract}

\begin{document}

\maketitle
\pagestyle{headings}
\setcounter{page}{1}
\pagenumbering{arabic}

\section{Introduction} 

Data-driven decision making systems generally try to maximize some domain-specific utility, which, if unconstrained, has the potential to damage disadvantaged groups in the process \citep{executive2016big}. To avoid this outcome, the study of fair treatment in automated decision making commonly focuses on ensuring that automated decisions follow the moral compass of an ideal society; this society being one in which decisions and societal outcomes are not based on immutable characteristics such as race or gender. Most past work in this area has focused on enforcing \emph{constraints} that provide some notion of fairness upon the decision making process, such as demographic parity (ensuring equal representation of groups in the decisions), equality of opportunity (ensuring equal true positive rates among different groups), or blindness (ensuring protected attributes are ignored entirely in the decision making process, even through surrogate features). Prior work shows that these notions are unfortunately fundamentally incompatible with each other in most situations \citep{kleinberg2016inherent}.  Furthermore, these safeguards are made based on a short snapshot of the process, namely, the model's immediate outcome.  In response, recent studies have looked into the impact of fair decisions over time, finding that some constraints that enforce equal treatment can actively harm disadvantaged groups.

As a running example that we use throughout this paper, consider the task of a bank deciding on individuals to give loans as a function of a given person's (estimated) probability of paying back the loan; this setting is common through the literature on fair decision making, though the same principles apply to a wide variety of other tasks, such as granting school admissions \citep{kleinberg2018algorithmic}, making predictions of criminal recidivism \citep{larson2016we}, and many others.  There are many competing incentives in this setting: banks are not incentivized to give loans that will not be repaid, but a historically disadvantaged population, without access to loans, may have a harder time re-establishing credit and improving the overall financial state of the group without receiving a loan.

In this paper, we consider the fair decision making setting as an iterative, repeated process, where e.g. loan granting decisions will have future effects on the probability of different groups to repay.  We then consider the impact of different fairness constraints in this setting.  In particular, we make three contributions in this work.  First, we propose and motivate a continuous population model, governed by the Beta distribution, that captures the entire population state in terms of propensity to pay back a loan.  We show than under this model, several common notions of fairness can be expressed through a simple analytic parameterization of the cumulative distribution function. 

Second, we propose a model of population \emph{dynamics} that captures how granting or denying loans impacts the population as a whole.  We show that lenders operating to maximize returns \emph{independently} for each population group can naturally lead to the ``Matthew Effect'' \citep{merton1968matthew}, summarized by the adage, ``the rich get richer, the poor get poorer'', where the final population means of different groups bifurcate based upon their initial state. Yet, we show that for a set of intuitive assumptions of population dynamics, \emph{any} constraint, according to our definition, that enforces equal treatment guarantees that groups converge to the same distribution; This suggests that at least that within our proposed model, the seemingly negative impacts of some notions of fair treatment suggested by previous work may not dominate the long-term well-being of all groups.  

Last, we address the problem of \emph{estimation} and stereotypes to show that when payback probabilities are incorrectly estimated among groups (and indeed, previous studies have found that automated systems are less effective on non-majority groups, often a result of the data collection process), there can still be bifurcations of the population, even under fair policies.  Although these results are largely dependent upon the assumptions we make regarding population distributions and dynamics, we find that this proposed model both reinforces prior work, that fair policies can lead to convergence, while coming to a different conclusion on the effect of stereotypes.  In contrast with previous work, we show that in this model differing levels of misestimation among groups, except in trivial cases, result in fair policies \emph{never} being able to bring groups to convergence. This suggests that models of population dynamics imply different outcomes for the effect misestimation, to a greater degree than has currently been explored.  We also assess the effects of these results using a publicly available data set of loan repayment probabilities for different demographic groups.

\section{Relations to Past Work}

Recent work on fairness in machine learning arose out of a growing concern for biases seeping into algorithmic procedures \citep{kleinberg2018algorithmic, larson2016we, bolukbasi2016man, buolamwini2018gender, crawford2016artificial}.  This phenomenon was highlighted by a report from the White House on the possible harms that biased classifiers and decisions can have on society as a whole \citep{executive2016big}. As a result of these alarms, researchers increased focus on methods for mitigating the influence that intentional and unintentional biases have on learned classifiers and decisions. Although the approaches vary, they can be generally structured into three classes of approaches, (1) fair preprocessing of data \citep{zemel2013learning} (2) fair post-processing of outputs \citep{hardt2016equality, kleinberg2016inherent} and (3) inherent fairness in models \citep{kusner2017counterfactual, raff2018gradient, zafar2015fairness, donini2018empirical}.

\paragraph{Approaches to Fair Learning} 
Through preprocessing methods such as omitting information that leads to biases \citep{grgic2016case} or by introducing new representations \citep{sattigeri2018fairness, madras2018learning} researchers have attempted to ensure that model inputs, and their resulting outcomes, have varying levels of independence from protected attributes. However \citet{datta2017proxy} show that proxy discriminators often remain in the data, as confounders for protected attributes are still a part of other features that are necessary for a model. Proxies can force a difficult choice of whether or not to leave features strongly correlated with protected attributes in the data or remove features that are necessary for model accuracy. To combat this effect, many researchers have introduced methods of ensuring that models are unable to learn biases as a part of training \citep{raff2018gradient, zhang2018mitigating}, or that the training itself is able to inherently conform to some notion of fair treatment or fair outputs for individuals \citep{manisha2018neural}. 
	
Societally, there is no accepted true notion of fair treatment. As a result, additional work had to be done to constrain the possibility of approaches so that well-meaning individuals are not led to dead-ends and impractical applications \citep{kleinberg2016inherent, pleiss2017fairness}. This area has found that fair constraints cannot generally be combined to make a decision more fair. Some ideas of fair treatment are incompatible with each other except in trivial cases or impossibly specific circumstances.

\paragraph{Impact of Fair Policies} 
The majority of prior work has been in a static setting in which the impact of these decisions are not studied, ensuring only that they conform to some notion of fair treatment. However, there is a growing body of work \citep{liu2018delayed, mouzannar2019fair, hu2018welfare} on the impact of fair policies and fair classifiers on the population. \citet{liu2018delayed} showed that even though many ``fair'' policies may sound beneficial, they can actually harm the population over time. For example, when applying demographic parity \citep{calders2009building,zliobaite2015relation} in a loan setting, if required representation proportition for both groups is too large, the policy will allow many people to receive loans they cannot repay, causing the group financial state to drop. Mouzannar et al., \citep{mouzannar2019fair} focused on a similar setting, and also show a similar yet counterintuitive result: that in some cases, entirely uncontrained policies lead to convergence of the groups.

\paragraph{Our Contributions} 
Here, we focus on the implications of fair policies on a population over time. We build upon work in \citep{liu2018delayed} and analyze the long-term equilibria of policies that enforce different selection rates among advantaged and disadvantaged groups.  However, extending the prior work of \citet{mouzannar2019fair}, which focused on discrete state settings and Bernoulli random variables, we assume a full continuous parameterization of the population success profiles, governed in our case by the Beta distribution.  Under this assumption,  we conveniently find that we can capture fair policies with a simple analytic parameterization of the cumulative distribution function.  We show that unconstrained policies can indeed reinforce inequalities amongst the two groups; however, we show that in this setting, enforcing \emph{any} fairness constraint will cause the two groups to converge in distribution.  However, in constrast to the setting of 
\citet{mouzannar2019fair}, we show that under our model, the effect of misestimations or stereotypes leads to a setting where populations do not converge. 
%Lastly, our model incorporates methods from dynamic programming and optimal control \citep{bertsekas1995dynamic}. These methods provide a process by which institutions can a priori determine the effects of policies that enforce equal treatment, while allowing these institutions to decide on more beneficial policies, or to explicitly weigh reward with population well-being.

\section{Problem Setting and Population Distributions}

When assessing whether to grant a loan, the quantity of interest to a lender is the conditional probability of an individual repaying a loan given relevant features such as credit score, income, or repayment history. Under our assumptions, we posit that loan repayment (at least at an observable level) is a \emph{fundamentally} stochastic concept, and so this conditional probability is not merely an estimation artifact but a genuine real-valued property of individuals. 

For this setting, our initial goal is to model the \emph{distribution} of these probabilities over the entire population. There are many possible distributions which model economic well-being of the population, however, given that we are restricted to the $[0,1]$ space of repayment probabilities, we choose to use a standard distribution over this space, the Beta distribution, to model the population; we show in a later section the fits of such distribution to real data.  We want to stress that we are \emph{not} claiming that this model perfectly captures real world effects, but rather that is a simple and intuitive setting where the effects of bias and misestimation can be illustrated.

\begin{assumption}
At time $t$, the distribution of payback probabilities for individuals from group $(i)$, denoted by $\pi_t^{(i)}(x)$, is governed by a Beta distribution
\begin{equation}
\pi^{(i)}_t(x) = \mathbf{Beta}(x; \mu_t, c)
\end{equation}
where $\mathbf{Beta}(x;  \mu_t, c )$ is the density of a Beta distribution in a ``mean-parameterized'' form, such that for standard beta distribution parameters, $a,b$, $\mu_t = \mathbb{E}[x] = \frac{ a }{ a + b }$, and $c$ is a shape parameter, $c = a + b$. (Note that this implies $a = c  \mu_t$ and $b = c(1-\mu_t)$).
\label{assumption:def}
\end{assumption}

In the multi-group setting, we consider two groups that have different means, $\mu_t^{(0)}$ and $\mu_t^{(1)}$, where $\mu_t^{(0)} < \mu_t^{(1)}$ implies that at time $t$, group $0$ is a disadvantaged group and $1$ is an advantaged group who share the same shape $c > 0$.  A bank grants loans to each group according to a threshold $A^{(i)}$ for $i \in \{0,1\}$, where individuals in group $i$ who are above this threshold receive a loan. In general, the threshold will typically be selected by a lender as a function of the current mean, such that $A^{(i)} = \tau(\mu^{(i)} )$ for policy $\tau : [0,1] \rightarrow [0,1]$.  Fair policies will typically specify $\tau$ such that some proportion $s$ (which could be the total selected population, the true positive, etc), will be equalized over the two groups.

% Different policies can enforce different notions of fairness, as formalized by enforcing statistical parity for some proportion $s$ This constraint $w( x )$ enforces a notion of statistical parity for proportion $s = w( x ) \int_{A^{(i)}}^{1} \pi_{t}( x ) dx$. The thresholding policy can (but is not required to) be constrained such that the thresholds for different groups correspond to a specific space of actions in which blindness, equality of opportunity, demographic parity, and/or other criteria are satisfied.

% Following prior work in \cite{liu2018delayed}, every fair policy of the form, $\langle w( x ), \pi \circ \tau \rangle$, where $w( x ) > 0$, and $w( x )$ is non-increasing in $x$ can be represented through the population's inverse cumulative distribution. In this Beta-population setting, similarly, most ``standard'' fairness policies have a convenient form resulting as a special case of a general class of policies.

An appealing property of the Beta characterization of population state is that several common fair policies have a simple and analytic form, as specified by the following proposition.
\begin{proposition}
Setting a threshold
\begin{equation}
A^{(i)} = \tau( s, \mu^{(i)} ) = I^{-1}_{1 - s}( \mu^{(i)} c + k_1, (1 - \mu^{(i)}) c + k_2)
\end{equation}
in which $I^{-1}$ is the inverse of the regularized incomplete beta function, $s$ describes the equal treatment proportion (proportion in demographic parity, true positive rate in equality of opportunity, etc), with
\begin{align*}
k_1 &> -\min\{c \mu^{(i)}, c \mu^{(j)}\} \\
k_2 &> - (c - \max\{c\ \mu^{(i)}, c\ \mu^{(j)}\}).
\end{align*}
Any such $\tau$ is a fair thresholding policy in which Equality of Opportunity, Demographic Parity, and Blindness are special cases.
\label{prp:policy}
\end{proposition}
Specifically, $k_1=0, k_2=0$ in proposition \ref{prp:policy} implies demographic parity, $k_1=1, k_2=0$ implies equality of opportunity, $k_1$ or $k_2 \rightarrow \infty$ implies blindness, and the intersection of equality of opportunity and $k_1=0, k_2=1$ implies equalized odds. We provide a derivation of this proposition in Appendix \ref{proof:prp:policy}.

For the remainder of this paper, the function, $\tau$ denotes a thresholding policy as defined by the above proposition. We claim that all $\tau$ enforce some form of fair treatment, as this set comprises non-decreasing, right-continuous functions of group repayment probabilities, where any group with a lower average repayment probability receives a lower acceptance threshold. In other words, under this set of policies, a disadvantaged group will never have more strict requirements for a loan application than an advantaged group.

It should be noted that this proposition is essentially a specialization of \citep[Section 6]{liu2018delayed}, which shows that for a general class of distributions, many fair policies can be expressed by terms involving the inverse cumulative distribution function; the appealing element to the Beta distribution setting is that all of these policies have a simple analytic form governed by the two parameters mentioned above.

%Figure \ref{fig:poldynamics} shows the allowable thresholds for $A^{(i)}$ and $A^{(j)}$ permitted by these policies.

%\begin{figure}
%	\centering
%	\includegraphics[height=1.3in]{../Figures/Policies.pdf}
%	\includegraphics[height=1.3in]{../Figures/StartingExample.pdf} 
%	\includegraphics[height=1.3in]{../Figures/EvolvedExample.pdf}
%	\label{fig:poldynamics}
%	\caption{(left) shows an example parametric plot of the fair thresholds for an advantaged and disadvantaged group in $\tau$. We then show how two population distributions (center) evolve under the demographic parity policy for $\tau=0.7$ (right). }
%	\vspace{-0.2in}
%\end{figure}

\section{Dynamics}

We now introduce our model of population evolution under loan-granting policies.   In general, any Markovian evolution of the population distribution $\pi$ (for now we drop the group superscripts for notational clarity) will evolve over time according to 
\begin{equation}
\pi_{t+1}( x_{t+1} ) = \int \pi_{t}( x_t) F( x_{t}, x_{t+1} ) dx_{t},
\end{equation}
where the transition kernel $F : [0,1] \times [0,1] \rightarrow \mathbb{R}_+$ governs the evolution of the state. This is essentially the simple continuous-time version of state evolution considered e.g., in \citet{mouzannar2019fair}.  Although fully general, this form is not particularly useful in practice, because most easily-parameterized transition kernels would lead to subsequent states no longer characterized as a Beta distribution.  For this reason, we forgo the general transition kernel dynamics, and concentrate on a special case where \emph{the dynamics are governed by a simple update to the distribution parameters directly} (we later relate this to a specific form of transition kernel).

% where the transition kernel, $F$, updates the distribution in parameter space. As $c$ in our model is fixed, we define our dynamics as an update only over the mean, $\mu_{t}$. Depending on the results of a decision (i.e., the number of loans granted, and the resulting payback), the population moves toward a set of parameters that maximize the likelihood of this repayment event in response. 

The hypothesis that underlies our model is that granting loans (when they are repaid) has the potential to produce positive upward mobility for a population, while granting loans that are not repaid \emph{or} failing to grant loans at all can produce downward mobility. In more detail: we treat a loan being granted as a net benefit for its recipient, similarly to real-world influence in which small business loans open up the opportunity for financial growth through personal enterprises and mortgage loans allow upward social mobility through better neighborhoods, schools, and housing investments. Conversely not receiving a loan, or receiving a loan and failing to pay it back, has the potential for stagnation and decline in opportunities, education, and investments.

% This assumption structures our dynamics so that those who receive a loan increase their future repayment probability. Repaying a loan allows one to receive the full benefit of the loan, while defaulting causes a $\lambda \in (-1,0)$ penalty for their future repayment probabilities. This implies that the average population mean moves towards a penalized mean of those who received the loan. This penalty is captured by introducing an additional factor, $\beta$, where,
% \begin{equation}
% \beta \int x\ \pi( x ) dx = \int ( x + (1 - x) \lambda )\ \pi( x ) dx
% \end{equation} 
% $\beta$ captures the $\lambda$ penalty over the aggregate population. Additionally, not receiving a loan should decrease the mean to some nominal term, where if no one receives a loan, the population stabilizes to a mean, $\nu$.

Specifically, given the discussion above, we formalize our dynamics model according to the update to the mean parameterization of the distribution over time
\begin{assumption}
The dynamics of an arbitrary group distribution $\pi_t$ are governed by the following update of the mean $\mu_{t+1}$ (with $c$ remaining constant over time)
\begin{equation}
\mu_{t+1} = f( A, \mu_t) = \beta \cdot p_+( A, \mu_{t}) \cdot \mu_+( A, \mu_{t}) + 
	\nu \cdot ( 1 -  p_+( A, \mu_{t}))
\end{equation}
where $\nu \in [0,1]$, $\beta \in [0,1]$ are free parameters, $p_+$ denotes the proportion of the population above the threshold $A$ that is chosen by the lender and $\mu_+$ are the estimated parameters of the current repayment pattern; in this single parameter model, $\mu_+$ is the expected repayment probability of the distribution for those who receive a loan
\begin{equation}
p_+(A, \mu_t) = \int_A^1 \pi_t( x ) dx, \;\; \mu_+(A, \mu_t) = \frac{ \int_A^1 x\ \pi_t( x ) dx }{ p_+( A, \mu_t ) }.
\end{equation}
\label{assump:dynamics}
\end{assumption}

This update function, $f$, captures the following intuition: First, for the proportion of the population that receives the loan $p_+(A,\mu_t)$, the parameter that governs the next state, $\mu_{t+1}$, will move towards the mean of the population selected for a loan $\mu_+(A, \mu_t)$, scaled by the parameter $\beta$. Setting $\beta < 1$ captures the fact that there is an asymmetry between the harm of failing to repay and the benefit of repayment, with the consequences of failing to repay typically seen as \emph{more harmful} than the benefits of repaying.  These dynamics describe the positive or negative feedback effects of the community, where successfully repaid loans can improve the population, but failure to repay loans can damage the population.  Second, for the proportion of the population that does \emph{not} receive a loan, $1-p_+(A,\mu_t)$, the population will shift toward the nominal mean $\nu$; this captures the negative effects of being denied access to credit.

We want to emphasize that, like the choice of Beta distribution to begin with, we are not arguing that this model necessarily fits any particular data set.  Rather, we seek a general and tractable model that intuitively characterizes the effects of giving or not giving loans to a population, where we can analyze the effects of fair policies.  

Under the general transition kernel formulation above, this model is essentially positing the transition kernel
\begin{equation}
\begin{split}
\pi_{t+1}( x_{t+1} ) & = \int \pi_{t}( x_t) F( x_{t}, x_{t+1} ) dx_{t} \\ & = \int \pi_{t}( x_t) \frac{\mathbf{Beta}(x_{t+1};\mu_{t+1},c)}{\mathbf{Beta}(x_t;\mu_t,c)} dx_{t}.
\end{split}
\end{equation}

We believe that this transition kernel is a necessary assumption; if we choose a transition kernel not based on the beta distribution as above, we would break our assumption \ref{assumption:def} after a single time step. In order to model dynamics of a specific probability distribution,this requirement naturally comes about. 

%\begin{claim}
%For some fixed probability distribution at time $t$, $\pi_{t}( x_t )$, under the transition kernel, 
%\begin{align*}
%\pi_{t+1}( x_{t+1} ) & = \int \pi_{t}( x_t) F( x_{t}, x_{t+1} ) dx_{t},
%\end{align*}
%$\pi_{t+1}( x_{t+1} )$ is no longer the same class of probability distribution as $\pi_{t}( x_t )$.
%\end{claim}
%
%\begin{proof}
%
%\end{proof}

While this is obviously a very specialized setting, we argue that it captures many intuitive properties of actual lending environment, and it is thus illustrative to consider the impact of fair policies and (mis)estimation in this setting.

\paragraph{Understanding Fair Policies} Under this model, we want to determine whether or not a fair policy is equalizing over time: does it move groups that follow the \emph{same} dynamics (i.e., the same $\beta$ and $\nu$ terms) but \emph{different} initial conditions towards the same distribution?   Should a policy equalize populations, following \cite{hu2018short}, the fair class of ``group-aware'' policies can be considered as a short-term intervention allowing for `group blind'' long-term policies on the population as a whole. We formalize a policy that leads to equilibria as follows.
\begin{definition}
A thresholding policy, $\tau( \mu )$, is said to be asymptotically equalizing over two groups if for dynamics 
\begin{equation}
\mu^{(i)}_{t+1} = f(\tau(\mu^{(i)}_{t}), \mu^{(i)}_{t}), \;\; i \in \{ 0,1 \}
\end{equation}
then the means converge in the limit regardless of their initial values
\begin{equation}
\lim_{\substack{t \rightarrow \infty} } |\ \mu_{t}^{(0)} - \mu_{t}^{(1)} | = 0.
\end{equation} 
\label{Asy:EQ}
\end{definition}

\subsection{The Equilibrium Points of Fixed Policies}
Our first class of results shows that for any \emph{fixed} policy (that is, $\tau(\mu_t) = A_0, \forall \mu_t$), there is a unique equilibrium point $\mu_\infty$.  Such a policy may be beneficial to both groups, beneficial to one and harmful to the other, or harmful to both (here beneficial means that the group ends at a mean greater than their original mean, $\mu_\infty > \mu_0$, and harmful means the opposite).

We begin by considering $\nu$ and $\beta$ from assumption \ref{assumption:def}, and observing how these free parameters describe the equilibrium of certain specific, fixed policies.  By fixing $\tau( \mu ) = 1$, we deny all individuals from receiving loans, thus pushing the population mean to $\nu$.  Additionally, by fixing $\tau( \mu ) = 0$, we accept all individuals; this policy pushes the mean $\mu_{t+1}$ to $\beta \mu_t$, which approaches $0$ in the limit.  Thus, these two fixed policies each have an equilibrium, where 
\begin{equation}
f( A, \mu_\infty ) = \mu_\infty,
\label{eq:muinfty}
\end{equation} 
that attracts every $\mu$ toward $\mu_\infty$. For other values of $A$, by solving equation \ref{eq:muinfty}, we find that there exists an attracting $\mu_\infty$ equilibrium for \emph{every} threshold $A$ that satisfies definition \ref{Asy:EQ}. In addition, we show in Appendix \ref{proof:prop:EQ} that these stable equilibria are unique for all thresholds, $A$, that come from a fixed policy.

\begin{theorem}
Under the aforementioned dynamics, every fixed threshold $A_0$ has a single unique equilibrium point.  Furthermore, this equilibrium point is stable, in that the iteration $\mu_{t+1} = f(A_0,\mu_{t})$ will converge in the limit to this equilibrium point.
\label{prop:Eq}  
\end{theorem}

In general, there is no closed form expression for the fixed point $\mu_\infty$, but it can easily be computed numerically by simply finding a root of the one-dimensional nonlinear function \eqref{eq:muinfty}.

Given the existence of an equilibrium point of any fixed threshold, for two initial group means $\mu_0^{(0)}$ and $\mu_0^{(1)}$, we can characterize certain fixed threshold policies as either beneficial to both groups, beneficial to one, or harmful both, depending on the relation of $\mu_\infty$ to $\mu_0^{(0)}$ and $\mu_0^{(1)}$.
\begin{definition}
The curve of fixed points for two groups is separated to 3 partitions, \\
1) A positive equilibria is a fixed point according to Definition \ref{Asy:EQ}, where for $\mu_t^{(i)}$ and $\mu_t^{(j)}$ at time $t=0$, $\mu_{\infty} \geq \max\{ \mu_0^{(i)}, \mu_0^{(j)}\}$. \\
2) A negative equilibria is a fixed point  where for $\mu_t^{(i)}$ and $\mu_t^{(j)}$, $\mu_{\infty} \leq \min\{ \mu_0^{(i)}, \mu_0^{(j)}\}$. \\
3) A mixed equilibria is a fixed point, where for $\mu_t^{(i)}$ and $\mu_t^{(j)}$, $\mu_0^{(i)} < \mu_{\infty} < \mu_0^{(j)}$.
\label{def:EQcat}
\end{definition}

Finally, we can characterize the ``maximum social welfare'' policy as the choice of $\tau(\mu_t) = A_0$ that achieves a maximal value of $\mu_\infty$ (and which thus must be beneficial to both groups).  This is characterized by the following proposition.
\begin{proposition}
Under a fixed policy $\tau( \mu ) = A_0, \forall \mu$, the policy that leads to the maximum equilibrium point is a fixed threshold at $\tau( \mu ) = \frac{\nu}{\beta}$.
\label{prp:socialwelfare}
\end{proposition}

\begin{figure}
	\includegraphics[scale=0.7]{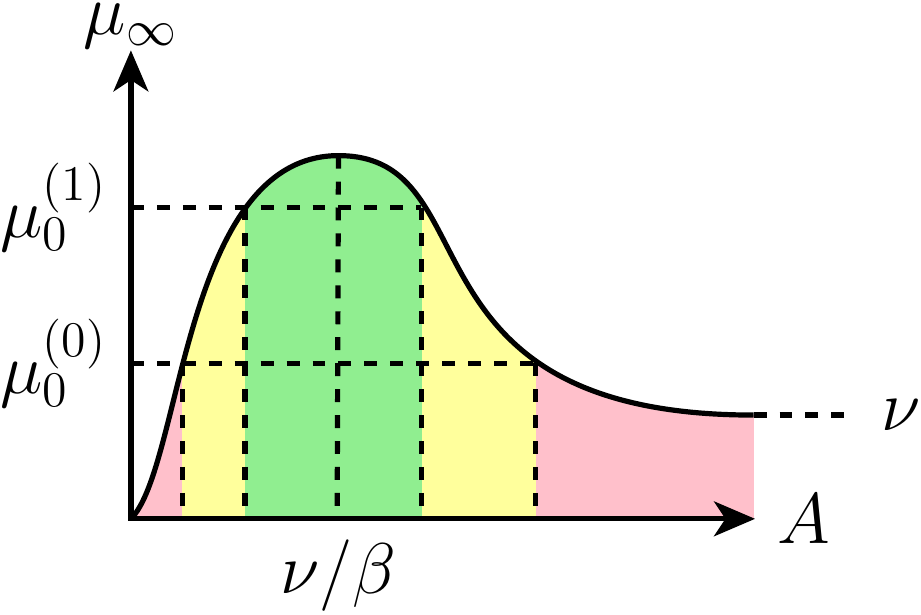}
	\caption{ We show an example of this $\mu_{\infty}$ equilibrium curve where each $A$ has a unique equilibrium point that attract $\mu$. For advantaged $\mu_0^{(i)}$ and disadvantaged $\mu_0^{(j)}$, the green, yellow, and red section shows thresholds that lead to positive, mixed, and negative equilbria respectively.}
	\label{figure:EquilibriumCurve}
\end{figure}

Putting these points together, we can view the possible outcomes for different fixed policies in relation to the group initial states, shown in Figure \ref{figure:EquilibriumCurve}.  This curve roughly parallels the outcome curve from \cite{liu2018delayed}, but with the important distinction that it refers to \emph{infinite horizon} group outcomes rather than single-step outcomes.

\subsection{Institution Rewards and Optimal Policies}

Although the existence of fixed points in the dynamics for fixed policies is an interesting feature of the model, in most cases a lender would want to decide which threshold to use (or even change threshold dynamically), so as to maximize some utility over time.  In this section, we therefore introduce an institution utility, and characterize the optimal policies that result from the standard Bellman equations for this setting.  The first main takeaway from this section is that, under the dynamics model we introduced above, if a lender is allowed to follow an optimal, unconstrained policy for each group independently, then this may lead to the Matthew effect, given as a bifurcation where groups with an initial mean state $\mu_0$ over some threshold will see their means increase, while groups with an initial state below this threshold will see their means decrease.  However, the second takeaway from this section is that if we further restrict the class of policies to jointly obey \emph{any} fairness constraint defined by proposition \ref{prp:policy} over the two groups, then the means of the two groups will always converge.

\paragraph{Institution Utility}
In the loan-granting setting, if a bank loans one unit of money to an individual, the bank expects some interest $R$ on their investment, while if an individual cannot repay their loan, then the bank loses this unit of money. We thus introduce a reward function that captures this specific utility, $g$, for the bank as,  
\begin{equation}
	g( A, \mu ) = p_+( A, \mu) (( 1 + R ) \mu_+(A,\mu)  - 1))
\end{equation}

The reward function models the intuition above, where a bank seeks to maximize its utility by ensuring the largest fraction of the population receives a loan that generates profit for the bank. Typically, the bank seeks to maximize its utility by selecting an interest, $R$, that is not too high, which would force people to fall into delinquency, but not too low, which unnecessarily decreases their profit.

\paragraph{Optimal Policies via the Bellman Equation}
We can characterize the policies which maximize the discounted infinite horizon return via the standard Bellman equation
\begin{equation}
J_\gamma^\star(\mu) = \max_{A} \left \{g(A, \mu) + \gamma J_\gamma^\star(f(A, \mu)) \right \}
\end{equation}
where $J^\star_{\gamma}(\mu)$ denotes the optimal value function (the expected sum of discounted reward under the optimal policy), and $\gamma \in [0,1)$ denotes a discount factor that trades off between immediate and future reward.  We also use the notation $\tau^\star_\gamma(\mu)$ to characterize the optimal policy, which is simply the threshold $A$ that achieves this maximum. 

As with the fixed point itself, there is no closed form expression for the optimal value function or policy.  However, because the quantity only involves a two-dimensional function over $A$ and $\mu$, we can easily compute the optimal function numerically via spline approximation and dynamic programming.  Furthermore, we can easily characterize the 
one-step greedy policy as, $\gamma = 0$.
%two extremes of the optimal policy, namely as we consider the fully greedy policy $\gamma = 0$, and the limit of the infinite horizon policy, as $\gamma \rightarrow 1$.

\begin{proposition}
 The one-step greedy policy for $\gamma = 0$ is given by the fixed policy
 \begin{equation}
\tau^\star_0(\mu) = \frac{1}{1+R}.
\end{equation}
%Further, in the limit as $\gamma \rightarrow 1$, the optimal policy goes to the fixed maximum social welfare policy
%\begin{equation}
%\tau^\star_\gamma(\mu) \rightarrow \frac{\nu}{\beta}.
%\end{equation}
\label{prop:gamma}
\end{proposition}

\paragraph{Bifurcations Under Unconstrained Optimal Policies}
%These two limiting policies are unique in that they do not actually depend on the state $\mu$, and thus, when applied to two groups with different initial states $\mu_0^{(0)}$ and $\mu_0^{(1)}$, would lead both to converge to the same mean.  However, for intermediate values $0 < \gamma < 1$, the resulting policy is \emph{not} fixed, and in practice can lead to a bifurcation of final state depending on the initial state of the different group populations.  This point is highlighted by the following example.

This policy is unique as it does not depend on the state, $\mu$, and thus, when applied to two groups with different initial states, $\mu_0^{(0)}$ and $\mu_0^{(1)}$, would lead both to converge to the same mean. However, as time passes ($0 < \gamma$), the resulting policy is \emph{not} fixed, and in practice can lead to a bifurcation of the final state depending on the initial state of the different group populations. This point is highlighted by the following example.

\begin{figure}
	\includegraphics[scale=0.5]{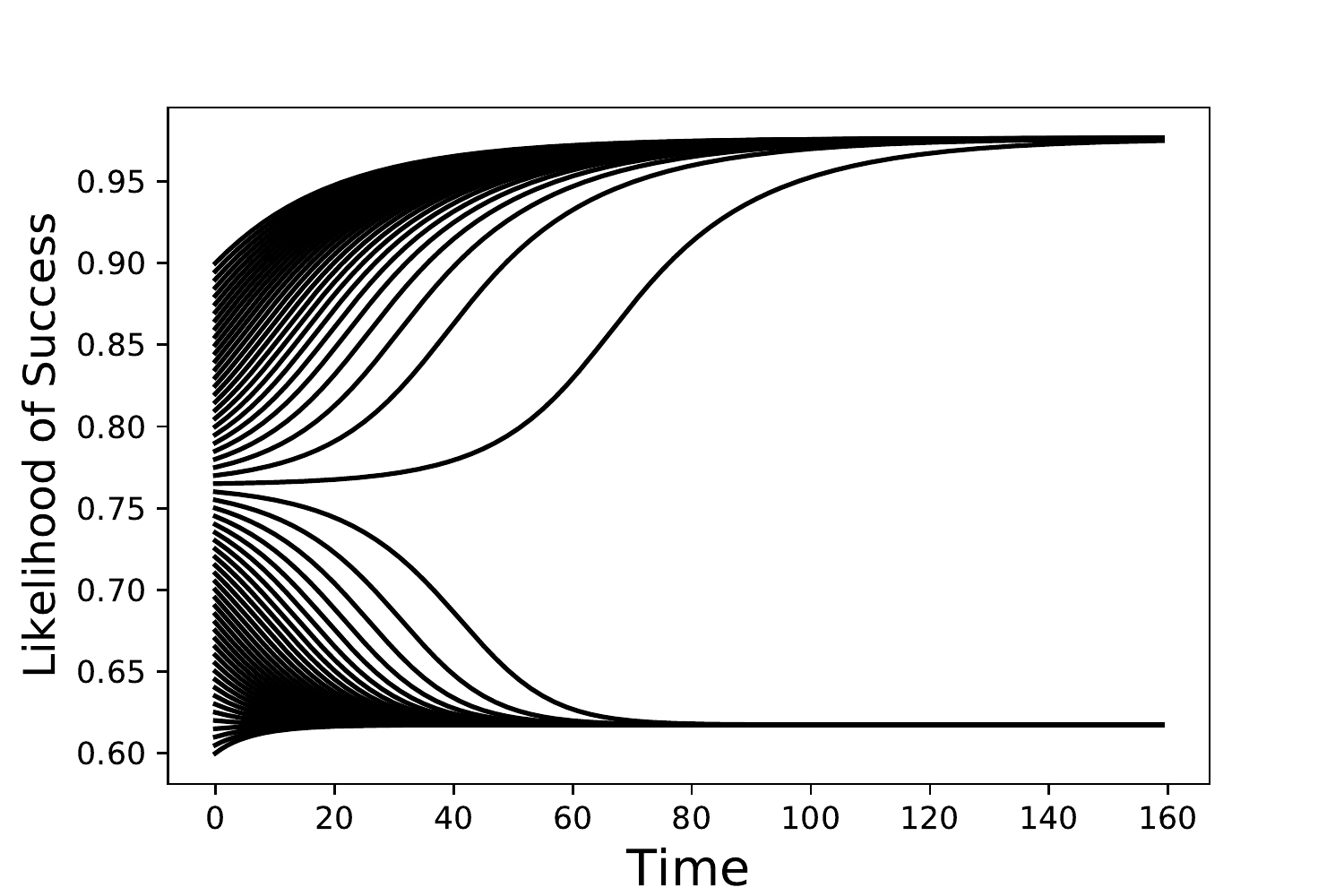}
	\caption{An illustration of the Matthew Effect, wherein a bifurcation of population well-being leads to individuals or groups who benefit from a system continuing to benefit and individuals or groups who struggle are unable to reap the same rewards.}
	\label{fig:matthew}
\end{figure}

\begin{example}
Consider the case of the dynamics specified above with $c = 1.6$, $\nu = 0.2$, $\beta = 0.99$, institution reward given by $R = 0.25$ and discount, $\gamma = 0.6$.  Then the all initial means with $\mu_0 \geq 0.76$ converge to a final mean of $\mu_\infty = 0.976$ while all initial means with $\mu_0 < 0.76$ converge to a final mean of $\mu_\infty = 0.617$.  The illustration of this effect is shown in Figure \ref{fig:matthew}.  Thus, for two groups with $\mu_0^{(0)}$ below 0.76 and $\mu_0^{(1)} > 0.76$ the means of the two groups will never converge.
\end{example}

This bifurcation results in the colloquially-named, ``Matthew Effect" where the advantaged group improves their state, while the disadvantaged group is never able to reach the same level of growth.  This effect was originally coined as a result of well-known scientists receiving a disproportionate level of credit in their collaborations or through the independent discoveries of multiple researchers. The early advantages of one individual or group will often lead to a multiplicative effect where they receive a greater return in their future endeavors, while those without this early advantage will require much more time or support to be able to reach the same heights \citep{crystal1990cumulative, pacheco2008political, willson2007cumulative}.

For an institution to achieve the optimal reward for sufficiently high discount factor, $\gamma$, the model causes the bank to operate at a loss so as to improve the financial state of the largest proportion of people. As this proportion's state improves, the bank recoups its losses returning a greater utility than either the utility maximizing policy or the social good policy. However, for a given time horizon, if the bank cannot recoup its losses by improving the state of a group sufficiently, the policy disregards this group and allows them to settle at different equilibria.

\paragraph{Convergence of the Optimal Policy Under Fairness Constraints}
A key feature in the above setting is that the lender was free to impose the same ``optimal'' policy on the state of each of the two groups independently, with no fairness constraint.  This in fact leads to a situation where an individual from disadvantaged group was required to have a much \emph{higher} probability of repayment in order to receive a loan than an individual in the advantaged group.  In contrast, if we require that the joint policy over both groups obey any fairness constraint defined by proposition \ref{prp:policy}, then this situation cannot occur.  Indeed, as we show in this section, if the joint policy (note that this can be \emph{any} policy, not just the optimal one) is constrained to obey \emph{any} fairness constraint as defined in Proposition \ref{prp:policy}, then under these same dynamics both groups will converge to the same fixed point.

Before providing a theorem on the equilibria of fair policies, we detail the implication that institutional rewards have on the allowable policy . 
\begin{lemma}
For all rewards, $R \leq \frac{\beta}{\nu} - 1$, the optimal control model is never incentivized to set the threshold less than $\frac{\nu}{\beta}$.  
\label{lma:reward}
\end{lemma}
For small rewards, it may be possible for a loan-granting entity to focus more on the advantaged group that provides a higher profit. As the reward increases, the institution is more likely to provide a loans to disadvantaged populations only if it causes no more harm than a policy with a lower reward.

Constraining our reward, $R \leq \frac{\beta}{\nu} - 1$, we show whether fair policies in this model will lead to convergent equilibria for groups. In essence, these policies enforce an ordering such that if $\mu^{(j)} > \mu^{(i)}$, then $A^{(j)} \leq A^{(i)}$. By enforcing threshold $R \leq \frac{\beta}{\nu} - 1$, and by extension $A \geq \frac{\nu}{\beta}$, the model focuses on the cases in which institutions set a reward that is based upon a reasonable interest, $R$. If any institution wants to guarantee that they will not enforce a bifurcation among demographics, we will show that there must be some notion of fair treatment among groups. In addition, by disallowing a bifurcation, the institution has the potential to increase its long-term utility over a group-blind, unconstrained policy.

\begin{theorem}
Under the optimal thresholding policy for the described dynamics, with reward $R \leq \frac{\beta}{\nu} - 1$, under any fairness constraint, $\tau$, the two groups will always reach parity.
\label{thm:parity} 
\end{theorem}

\begin{proof}
   We separate fairness constraints into two cases. blind policies and group fairness policies.
   
   \paragraph{ Blind Policy } Under a blind policy, in which both groups always have the same threshold, as there is a single unique $\mu_{\infty}$ for each threshold, both groups are drawn to the same $\mu_{\infty}$. Thus, both groups will always reach parity.
   
   \paragraph{ Group Fairness Policies } For any group's threshold under a fixed $\tau$, $A^{(i)} = I^{-1}_{1-\tau}( \mu^{(i)} c + k_1, c - \mu^{(i)} c + k_2 )$, \ref{lma:reward} states that the optimal policy for each group will be $\tau$ such that $A^{(i)} \geq \frac{\nu}{\beta}, \forall\ i$. $\mu_{\infty}$ is strictly decreasing for $A \geq \nu$. 
   
   Under group fairness policies, if $\mu^{(1)} \leq \mu^{(2)}$, due to the inverse beta regularized function being a monotonically increasing function with respect to $\mu$, threshold $A^{(1)}$ is guaranteed to be less than or equal to $A^{(2)}$. As $\mu_{\infty}$ is strictly decreasing for all $A \geq \frac{\nu}{\beta}$, then $\mu_{\infty}^{(1)} \geq \mu_{\infty}^{(2)}$. If $\mu^{(1)} \leq \mu^{(2)}$ and $\mu_{\infty}^{(1)} \geq \mu_{\infty}^{(2)}$, then as each $\mu$ steps to their individual fixed points, they must intersect. When they intersect, the population reaches parity and from that moment on, both groups share the same fixed point.
\label{Proof:thm:parity}   
\end{proof}

\section{Dynamics Under Stereotypes and Misestimation}

Up to this point, we have been focusing on the ideal setting in which the distribution of repayment probabilities is known exactly. However, in real-world scenarios, we have a degree of uncertainty in a distribution's ordering. In other words we have focused on the setting in which we know $\pi_t(x) = \mathbf{Beta}(x;\mu_t,c)$ exactly. However in practice we will likely not know the true parameters of this distribution, and even will a proper mean a difference between the estimated and true $c$ values can lead to a divergence between the ``true'' underlying distribution and the estimated one.   

For the simplest case, suppose that we predict probabilities at the limit of available features, ie. measuring based solely on the protected attribute. In this case, the probability of repayment reduces to a point-mass centered at the demographic's mean ability to repay. If the institution enacts a race-blind policy with threshold above the average of the disadvantaged group, then this group will revert to $\nu$, while the advantaged group evolves according to the dynamics above where $p_+ = 1$, and thus ultimately reverts to mean $\nu$ as well.

Any other degree of information will cause a population to evolve according to some combination of the true update from Assumption \ref{assump:dynamics}, and the group mean $\mu_t$. We thus introduce a mispecification parameter, $\alpha$, that encapsulates this uncertainty in a distribution's ordering. As stated above, in the limit where $\alpha = 1$, the policy is effectively selecting randomly, and the group declines according to $\beta \mu$. As more information is known, $\alpha \rightarrow 0$, then the population evolves according to the true dynamics above. We formalize a new evolution function,
\begin{assumption}
The dynamics of an arbitrary group distribution $\pi_t$ with policy $\tau(\mu_{t}) = A$, are governed by the following update for true mean $\mu_{t}$ and mispecification parameter, $\alpha$ (with $c$ remaining constant over time),
\begin{align}
\begin{split}
f( A, \mu_{t}) &= \beta \cdot P_+( A, \mu_{t}) \cdot ( ( 1 - \alpha ) \mu_+( A, \mu_{t} ) + ( \alpha\ \mu_t ) ) \\
&+ \nu \cdot ( 1 -  P_+( A, \mu_{t} ).
\end{split}
\end{align}
\end{assumption}

Generally for two demographics, as stated above, a high enough $\alpha$ has the potential to lead to a bifurcation, however there remains a possibility for two groups to reach equality as long as these groups have an equal amount of uncertainty.

\begin{figure}
	\includegraphics[scale=0.5]{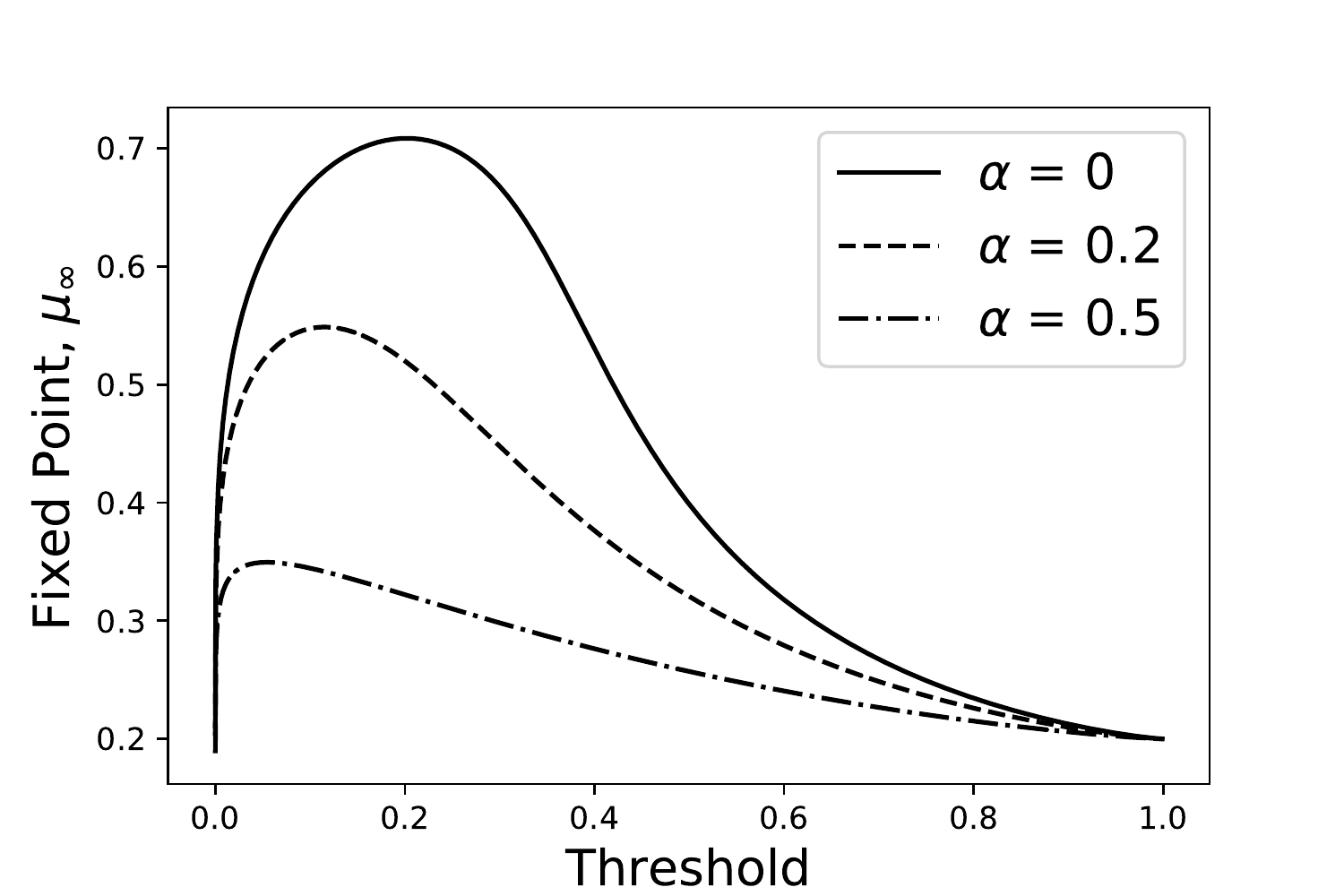}
	\label{fig:misestimate}
	\caption{Different levels of misestimation lead to thresholds that share completely distinct fixed points from one another. Here, $\nu = 0.2$, $\beta = 0.99$, $c = 2$} 
\end{figure}

\begin{theorem}
Under the optimal thresholding policy for two groups with equal shapes, $c$, thresholds, $A$, constrained above $\frac{\nu}{\beta}$ according to lemma \ref{lma:reward}, and any fairness constraint, $\tau$, as long as the two groups have equal $\alpha$ mispecifications, they will still reach parity. 
\label{thm:uncertainparity}
\end{theorem}

\begin{proof}
We again split this proof into a blind policy case and a group fairness case. Trivially, if $\alpha = 1$, then the group distributions are a point-mass at their current mean, so every thresholding policy will harm the population and force all groups to $\nu$. If $\alpha = 0$, then the problem is the same as Theorem \ref{thm:parity}. 

\paragraph{Blind Policy} $\mu_{t+1}$ is now a function of both the current true $\mu_{t}$, approximate optimal threshold, $A$, whereas in the previous case, $\mu_{t+1}$ was a function of only the true optimal threshold, $A$. With uncertainty, the previous equation changes to
\begin{align}
\begin{split}
f( A, \mu_{\infty} ) - \mu_{\infty} &= \beta\ ( 1 - \alpha ) P_+( A, \mu_{\infty} ) \mu_+( A, \mu_{\infty} ) \\
&+ ( 1 -  P_+( A, \mu_{\infty} ) )\ \nu \\
&+ \mu_{\infty} ( \beta P_+( A, \mu_{\infty} ) \alpha - 1 ) = 0.
\end{split}
\end{align}

This equation is identical to the earlier function for $\mu_{\infty}$ except that $\beta$ is scaled by $1 - \alpha$ and it adds the term $\mu_{\infty} ( \beta P_+( A; \mu_{\infty} ) \alpha - 1 )$ to the previous sum. This new term is always negative, so the peak of the $\mu_{\infty}$ curve is shifted closer to $A = 0$. In this new equation, there still remains a unique $\mu_{\infty}$ for each threshold $A$, so under a blind policy where both groups have the same threshold, they also share the same $\mu_{\infty}$. Thus, a blind policy still reaches equality.  

\paragraph{Group Fairness Policies} We find that $\mu_{t+1}$ is maximized at $A = \frac{\alpha \beta \mu - \nu }{ \beta( 1 - \alpha ) }$, so again the social welfare policy is a threshold here. Using lemma \ref{lma:reward}, the threshold $A$ under the optimal policy is always greater than or equal to $\frac{\alpha \beta \mu - \nu }{ \beta( 1 - \alpha ) }$. 

This expression $A = \frac{\alpha \beta \mu - \nu }{ \beta( 1 - \alpha ) } \in \mathbb{R}$ instead of $A = \frac{\nu}{\beta} \in [ 0, 1 ]$. As such, the optimal social welfare threshold must be the threshold, $A$, that minimizes the absolute value of the derivative of the $\mu_{t+1}$ with uncertainty, 
\begin{align}
\begin{split}
f( A,\mu_{t} ) &= \beta \cdot P_+( A, \mu_{t}) \cdot ( ( 1 - \alpha ) \mu_+( A, \mu_{t} ) + ( \alpha\ \mu_{t} ) ) \\
&+ \nu \cdot ( 1 -  P_+( A, \mu_{t} ) ).
\end{split}
\end{align}
If this derivative is negative for all $A \in [0,1]$, then no matter what thresholding policy is used, the total population will be harmed. Similarly, if the derivative is positive, then the population will be improved for all $A$. 

When the derivative does not have a root at any $A \in [0,1]$, every policy will update $\mu_t$ until there exists a root for some $A \in [0,1]$. Once $\mu^{(i)}$ for both groups reaches a point where the root exists for some $A$, then the dynamics from the previous proof take over. The optimal policy will always be greater than $\frac{\alpha \beta \mu - \nu }{ \beta( 1 - \alpha ) }$, so if $\mu^{(1)} < \mu^{(2)}$, then $A^{(1)} \leq A^{(2)}$, so $\mu_{\infty}^{(1)} \geq \mu_{\infty}^{(2)}$. Thus, the population $\mu$'s must converge to reach their respective equilibria.

It should be noted that greater levels of uncertainty $\alpha$ correspond to a greater influence of the current mean, so as $\alpha$ increases, $\max \mu_{\infty}$ decreases. So a greater level of uncertainty decreases the potential final mean of the population.
\end{proof}

Alternatively, if two groups have different $\alpha$, then we show below that they cannot reach equality.

\begin{theorem}
Under the optimal thresholding policy for two groups with equal shapes, $c$, thresholds, $A$, constrained above $\frac{\nu}{\beta}$, and any fairness constraint, if two groups do not have the same $\alpha$ mispecifications, they will never converge except in trivial instances. 
\end{theorem}

\begin{proof}
Given dynamics of the form, 
\begin{align}
\begin{split}
f( A, \mu_{t}) &= \beta \cdot P_+( A, \mu_{t}) \cdot ( ( 1 - \alpha ) \mu_+( A, \mu_{t} ) + ( \alpha\ \mu_t ) ) \\ 
&+ \nu \cdot ( 1 -  P_+( A, \mu_{t}),
\end{split}
\end{align}
the derivative of $f$ with respect to the mispecification, $\alpha$ is,
\begin{equation}
\frac{ \partial f }{ \partial \alpha } = \beta P_{+}( A, \mu_{t} ) ( \mu_{t} - \mu_{+} ( A, \mu_{t+1} ) ).
\end{equation}
This derivative is always negative for $A \neq 0$, so the function is strictly decreasing for non-trivial thresholds, $A \neq 0, 1$. As the dynamics are strictly decreasing with $\alpha$, the fixed point with respect to a specific threshold is strictly lower for any two uncertainty $\alpha$ parameters. 

In this case, where the mispecification parameters are not equal, there always exist a set of two thresholds in the system with an $\alpha_1$ that can map to a pair thresholds with $\alpha_2$, that share a fixed point, however this case is trivial as there is no guarantee that these thresholds will be chosen under a policy. This mapping only comes about as the range of function of $\mu_{\infty}$ with $\alpha_0$ and $\alpha_1 > 0$, is a subset of the range of every other function of $\mu_{\infty}$ with $\alpha_1$ and $\alpha_1 > \alpha_0 > 0$.
\end{proof}

It should be noted that the trivial cases that we mention above consist of a set of four thresholds that ensure convergence regardless of misestimation or initial distribution. This set is $\tau$ such that everyone from both groups are denied loans, which forces the mean of both groups to $0$; $\tau$ such that everyone from both groups are granted loans, which also forces the mean of both groups to $0$; the last two policies come out of the fact that if an institution is able to set a different acceptance threshold for every group, specific threshold can be found on either side of the peak that leads to the same long-term parameters. Regarding the last case, unless the thresholded proportion, defined bby $s$ in proposition \ref{prp:policy}, is specified to an incredibly specific and sufficiently large level of numerical precision, this case will not occur.

Prior work \cite{mouzannar2019fair} has similarly analyzed the effect of these stereotypes in a similar setting, however while dynamics that govern the Bernoulli distribution of repayment probabilities show that when a equal treatment constraint causes a decrease in loans to the advantaged group, and $\mu_t$ is underestimated for the disadvantaged group, regardless of whether the advantaged group is underestimated or overestimated, the populations will still reach equality. A second case, where a greater percentage of the disadvantaged group is granted a loan, with a negative bias on the advantaged group and a positive bias in the advantaged group also leads to equality. However, under our model, regardless of the acceptance rates for each group, as long as they do not have the same level of misestimation, they cannot reach equality. This implies to us that the effects of stereotypes is more model-specific than has previously been explored.

\section{Simulated Results}

%MAKE THIS PLOT BETTER

%\pgfplotsset{small}
\begin{figure}
%	\centering
	\includegraphics[scale=0.4]{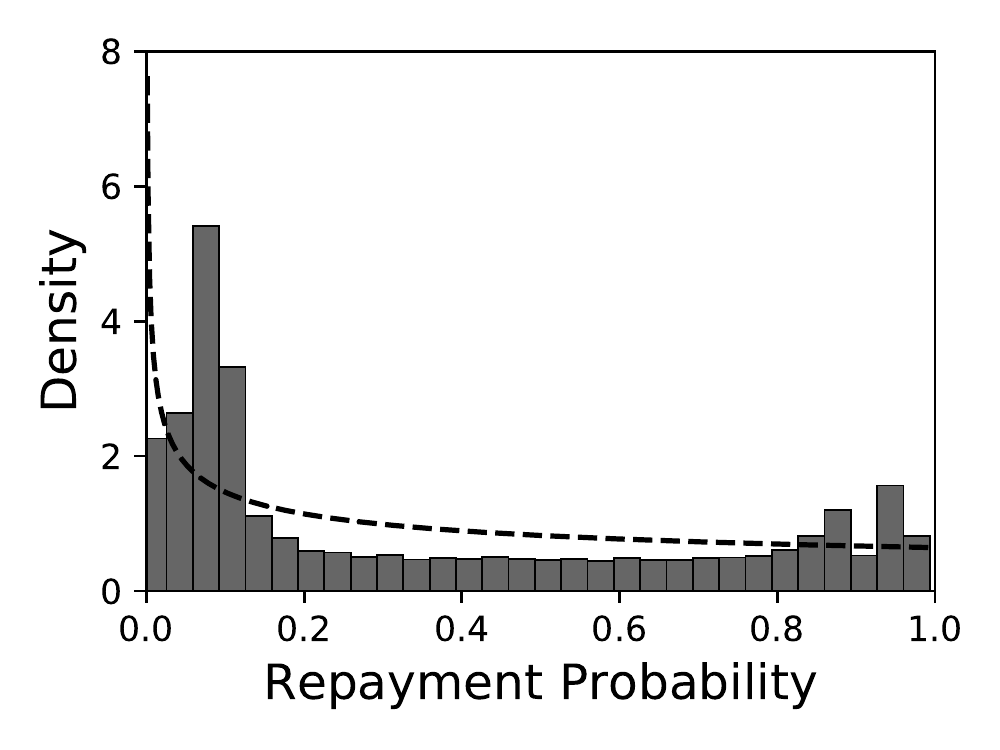} 
	\includegraphics[scale=0.4]{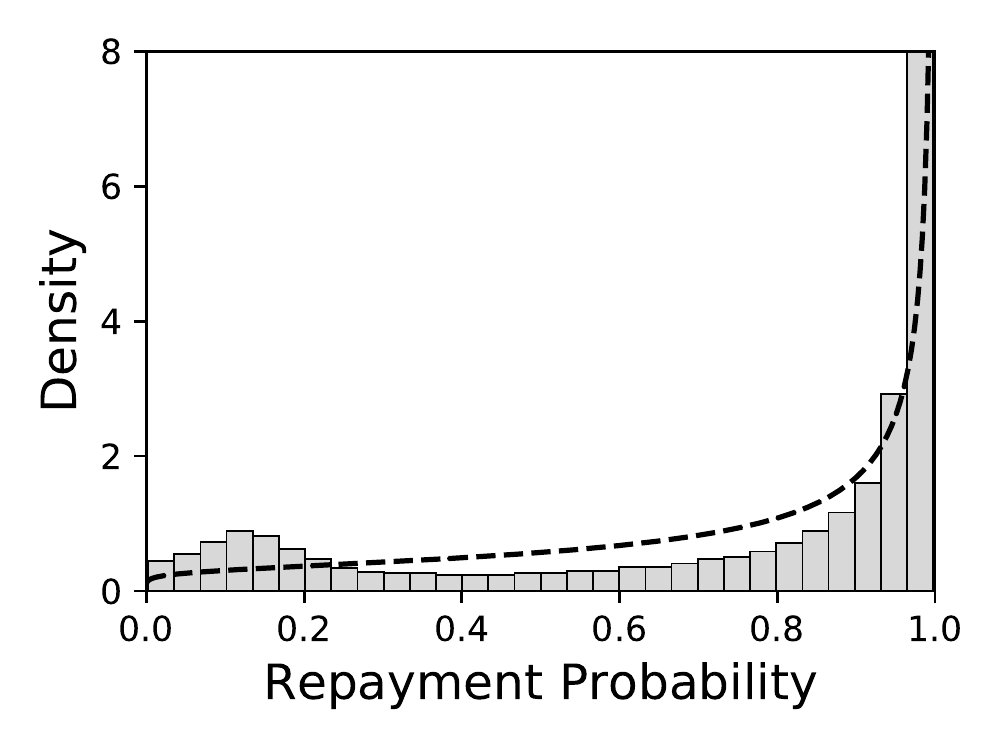}

	\includegraphics[scale=0.55]{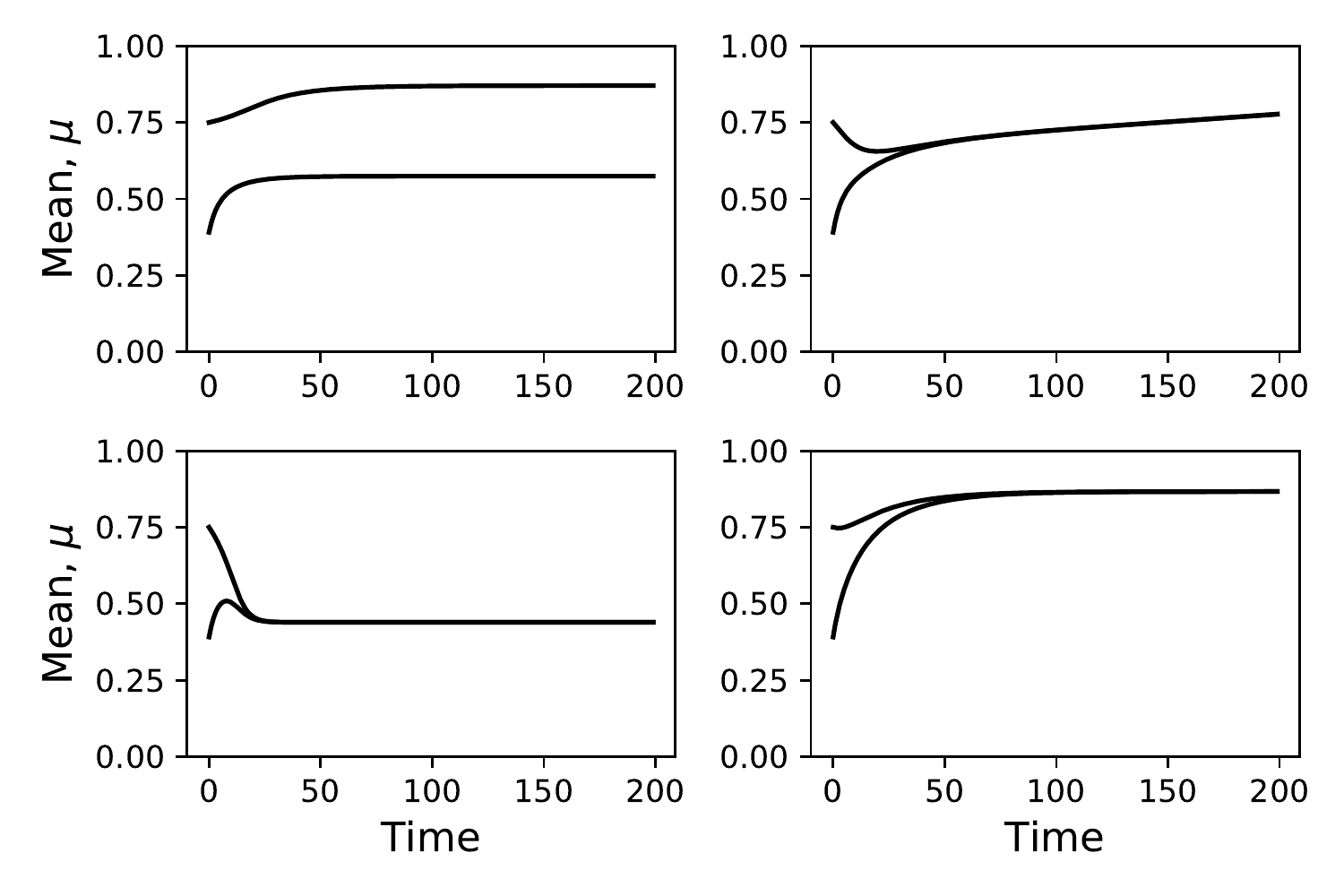}
	\caption{(Top) The histogram of repayment probabilities by race (Left = Black, Right = Non-Hispanic White). (Center, Left) The change in population wellbeing for individuals under an unconstrained optimal policy. (Center, Right) The change in population wellbeing for individuals under demographic parity. (Bottom, Left) The change in population wellbeing for individuals under a blind policy. (Bottom, Right) The change in population wellbeing for individuals under Equality of Opportunity.}
	\label{fig:real}
\end{figure}

We examine the impact of several fairness policies according to the dynamics introduced above. We simulate the evolution of populations on the data set of 301,526 preprocessed TransUnion TransRisk scores from \cite{reserve2007report,hardt2016equality}. This data set is separated by race including Black, Non-Hispanic White, Hispanic, Asian, and includes a cumulative distribution over credit scores of the population for each race and a distribution of 90-day delinquency rates over credit scores with respect for each race.  We process this information into a probability distribution of repayment probabilities in a three step process. We first calculate the derivative of a smooth interpolation of the given cumulative distribution in order to receive the probability of individuals in each race having a specific score. We then find the histogram of non-delinquency rates after 90 days with respect to the probability distribution of credit scores in Figure \ref{fig:real}. Both the histogram of delinquency rates and distribution of scores are plotted parametrically, in order to describe the repayment probability vs density per demographic.
 
After retrieving a histogram of repayment probabilities, we fit a Beta distribution to the graph and calculate $\mu$ and $c$ as in assumption \ref{assumption:def}. These parameters are used to find the optimal control policy for two races (Black, Non-Hispanic White) for each notion of fairness treatment. We then observe the evolution of $\mu$ under each policy. As stated previously, our model requires the shape $c$ to be equal for all groups, we fix the shape as the average of any two group shapes. In addition we assume a profit loss ratio for the lender of $R = 0.21$ and no uncertainty in the model, $\alpha = 0$. 
 
In the case of this simulated data, removing all fair constraints leads to a bifurcation; the advantaged group becomes much better off, while the disadvantaged group remains stagnant. Whereas adding constraints that enforce Equality of Opportunity improves both groups and leads them to positive equilibria. The blind policy leads groups to equilibria as well, however, it instead harms both groups, leading them to a worse off state than every other policy.

Yet, special attention should be drawn to the demographic parity example. Demographic Parity ultimately creates a positive equilibria, however, it first causes temporary harm  in the advantage population in order to lead both groups to a positive equilibria, where everyone ends up better of. While the policy causes active harm \cite{liu2018delayed} for one group, this harm does not dominate long-term dynamics, over time the policy leads to a better state than where the advantaged group began.

In the limit, policies that maximize group well-being, maximize reward, so a lender in the case presented here is better off under a fair policy. However, it is important to note that whether a specific fair policy remains optimal compared to an unconstrained optimal policy or another fair policy depends on the task at hand and initial conditions. Modeling the expected effects of these policies would allow institutions to determine the best policy for a given utility function.

\section{Discussion and Future Work}

This work considered the iterative processes of fair decision making.  We formalized a Beta-distribution-based population model, showed that several existing fairness conditions can be expressed as special cases of a general class of fair policies, and then analyzed a proposed dynamical update equation for this model.  We showed that under these conditions, although unconstrained policies can lead to bifurcations of different population means, \emph{any} fair policy according to our proposition \ref{prp:policy} will lead to a single equilibrium for both groups.  However, under the case of \emph{mis-estimated} repayment propabilities, these bifurcations can continue for two different groups.

Although these results naturally depend on the modeling assumptions, they highlight that fact that under reasonable conditions, ensuring that data-driven decision-making system accurately estimate group distributions is a major requirement for reaching equitable long-term world states. As historically underrepresented groups traditionally suffer from worse predictive performance, this is a important but worthwhile challenge.

\paragraph{Future Work} Here, we focused on one possible model for population dynamics. However, there are many equally or more valid models for modeling a population in a resource granting setting. Even in this case, we found that the behavior of stereotypes and misestimation is more model specific than we feel has been previously explored. A larger and more thorough explanation of how misestimation statistically affects outcomes on a population and ways of mitigating these effects is an important avenue for enforcing fair treatment among groups. 

Additionally, this work emphasized constraints that enforce fair policies, however, many methods of fair treatment may not have clear constraints, for more example, a decision being counterfactually fair \cite{kusner2017counterfactual}. As fair treatment is not universally defined, researchers have many opportunities of finding new ways to determining the impact of more abstract notions of fair treatment. 

Training classifiers to conform to a given notion of fair treatment, requires an understanding of whether or not these classifiers will actually be beneficial. As shown in our results, it is entirely possible for an unconstrained ``un-fair'' policy to lead to a better world-state than a tightly constrained fair policy. Just policies should not  only formally treats individuals equitably but ensure that decisions will result in equitable world-states for individuals who agree with and conform to a given idea of fair treatment. 

\bibliographystyle{ACM-Reference-Format}
\bibliography{citations}


\begin{thebibliography}{31}

%%% the Web address.

\ifx \showCODEN    \undefined \def \showCODEN     #1{\unskip}     \fi
\ifx \showDOI      \undefined \def \showDOI       #1{#1}\fi
\ifx \showISBNx    \undefined \def \showISBNx     #1{\unskip}     \fi
\ifx \showISBNxiii \undefined \def \showISBNxiii  #1{\unskip}     \fi
\ifx \showISSN     \undefined \def \showISSN      #1{\unskip}     \fi
\ifx \showLCCN     \undefined \def \showLCCN      #1{\unskip}     \fi
\ifx \shownote     \undefined \def \shownote      #1{#1}          \fi
\ifx \showarticletitle \undefined \def \showarticletitle #1{#1}   \fi
\ifx \showURL      \undefined \def \showURL       {\relax}        \fi
% invisible to TeX
\providecommand\bibfield[2]{#2}
\providecommand\bibinfo[2]{#2}
\providecommand\natexlab[1]{#1}
\providecommand\showeprint[2][]{arXiv:#2}

\bibitem[\protect\citeauthoryear{Bolukbasi, Chang, Zou, Saligrama, and
  Kalai}{Bolukbasi et~al\mbox{.}}{2016}]%
        {bolukbasi2016man}
\bibfield{author}{\bibinfo{person}{Tolga Bolukbasi}, \bibinfo{person}{Kai-Wei
  Chang}, \bibinfo{person}{James~Y Zou}, \bibinfo{person}{Venkatesh Saligrama},
  {and} \bibinfo{person}{Adam~T Kalai}.} \bibinfo{year}{2016}\natexlab{}.
\newblock \showarticletitle{Man is to computer programmer as woman is to
  homemaker? debiasing word embeddings}. In \bibinfo{booktitle}{\emph{Advances
  in neural information processing systems}}. \bibinfo{pages}{4349--4357}.
\newblock


\bibitem[\protect\citeauthoryear{Buolamwini and Gebru}{Buolamwini and
  Gebru}{2018}]%
        {buolamwini2018gender}
\bibfield{author}{\bibinfo{person}{Joy Buolamwini} {and}
  \bibinfo{person}{Timnit Gebru}.} \bibinfo{year}{2018}\natexlab{}.
\newblock \showarticletitle{Gender shades: Intersectional accuracy disparities
  in commercial gender classification}. In \bibinfo{booktitle}{\emph{Conference
  on Fairness, Accountability and Transparency}}. \bibinfo{pages}{77--91}.
\newblock


\bibitem[\protect\citeauthoryear{Calders, Kamiran, and Pechenizkiy}{Calders
  et~al\mbox{.}}{2009}]%
        {calders2009building}
\bibfield{author}{\bibinfo{person}{Toon Calders}, \bibinfo{person}{Faisal
  Kamiran}, {and} \bibinfo{person}{Mykola Pechenizkiy}.}
  \bibinfo{year}{2009}\natexlab{}.
\newblock \showarticletitle{Building classifiers with independency
  constraints}. In \bibinfo{booktitle}{\emph{2009 IEEE International Conference
  on Data Mining Workshops}}. IEEE, \bibinfo{pages}{13--18}.
\newblock


\bibitem[\protect\citeauthoryear{Crawford}{Crawford}{2016}]%
        {crawford2016artificial}
\bibfield{author}{\bibinfo{person}{Kate Crawford}.}
  \bibinfo{year}{2016}\natexlab{}.
\newblock \showarticletitle{Artificial intelligence’s white guy problem}.
\newblock \bibinfo{journal}{\emph{The New York Times}}  \bibinfo{volume}{25}
  (\bibinfo{year}{2016}).
\newblock


\bibitem[\protect\citeauthoryear{Crystal and Shea}{Crystal and Shea}{1990}]%
        {crystal1990cumulative}
\bibfield{author}{\bibinfo{person}{Stephen Crystal} {and}
  \bibinfo{person}{Dennis Shea}.} \bibinfo{year}{1990}\natexlab{}.
\newblock \showarticletitle{Cumulative advantage, cumulative disadvantage, and
  inequality among elderly people}.
\newblock \bibinfo{journal}{\emph{The Gerontologist}} \bibinfo{volume}{30},
  \bibinfo{number}{4} (\bibinfo{year}{1990}), \bibinfo{pages}{437--443}.
\newblock


\bibitem[\protect\citeauthoryear{Datta, Fredrikson, Ko, Mardziel, and
  Sen}{Datta et~al\mbox{.}}{2017}]%
        {datta2017proxy}
\bibfield{author}{\bibinfo{person}{Anupam Datta}, \bibinfo{person}{Matt
  Fredrikson}, \bibinfo{person}{Gihyuk Ko}, \bibinfo{person}{Piotr Mardziel},
  {and} \bibinfo{person}{Shayak Sen}.} \bibinfo{year}{2017}\natexlab{}.
\newblock \showarticletitle{Proxy non-discrimination in data-driven systems}.
\newblock \bibinfo{journal}{\emph{arXiv preprint arXiv:1707.08120}}
  (\bibinfo{year}{2017}).
\newblock


\bibitem[\protect\citeauthoryear{Donini, Oneto, Ben-David, Shawe-Taylor, and
  Pontil}{Donini et~al\mbox{.}}{2018}]%
        {donini2018empirical}
\bibfield{author}{\bibinfo{person}{Michele Donini}, \bibinfo{person}{Luca
  Oneto}, \bibinfo{person}{Shai Ben-David}, \bibinfo{person}{John~S
  Shawe-Taylor}, {and} \bibinfo{person}{Massimiliano Pontil}.}
  \bibinfo{year}{2018}\natexlab{}.
\newblock \showarticletitle{Empirical risk minimization under fairness
  constraints}. In \bibinfo{booktitle}{\emph{Advances in Neural Information
  Processing Systems}}. \bibinfo{pages}{2791--2801}.
\newblock


\bibitem[\protect\citeauthoryear{Grgic-Hlaca, Zafar, Gummadi, and
  Weller}{Grgic-Hlaca et~al\mbox{.}}{2016}]%
        {grgic2016case}
\bibfield{author}{\bibinfo{person}{Nina Grgic-Hlaca},
  \bibinfo{person}{Muhammad~Bilal Zafar}, \bibinfo{person}{Krishna~P Gummadi},
  {and} \bibinfo{person}{Adrian Weller}.} \bibinfo{year}{2016}\natexlab{}.
\newblock \showarticletitle{The case for process fairness in learning: Feature
  selection for fair decision making}. In \bibinfo{booktitle}{\emph{NIPS
  Symposium on Machine Learning and the Law}}, Vol.~\bibinfo{volume}{1}.
  \bibinfo{pages}{2}.
\newblock


\bibitem[\protect\citeauthoryear{Hardt, Price, Srebro, et~al\mbox{.}}{Hardt
  et~al\mbox{.}}{2016}]%
        {hardt2016equality}
\bibfield{author}{\bibinfo{person}{Moritz Hardt}, \bibinfo{person}{Eric Price},
  \bibinfo{person}{Nati Srebro}, {et~al\mbox{.}}}
  \bibinfo{year}{2016}\natexlab{}.
\newblock \showarticletitle{Equality of opportunity in supervised learning}. In
  \bibinfo{booktitle}{\emph{Advances in neural information processing
  systems}}. \bibinfo{pages}{3315--3323}.
\newblock


\bibitem[\protect\citeauthoryear{Hu and Chen}{Hu and Chen}{2018a}]%
        {hu2018short}
\bibfield{author}{\bibinfo{person}{Lily Hu} {and} \bibinfo{person}{Yiling
  Chen}.} \bibinfo{year}{2018}\natexlab{a}.
\newblock \showarticletitle{A short-term intervention for long-term fairness in
  the labor market}. In \bibinfo{booktitle}{\emph{Proceedings of the 2018 World
  Wide Web Conference on World Wide Web}}. International World Wide Web
  Conferences Steering Committee, \bibinfo{pages}{1389--1398}.
\newblock


\bibitem[\protect\citeauthoryear{Hu and Chen}{Hu and Chen}{2018b}]%
        {hu2018welfare}
\bibfield{author}{\bibinfo{person}{Lily Hu} {and} \bibinfo{person}{Yiling
  Chen}.} \bibinfo{year}{2018}\natexlab{b}.
\newblock \showarticletitle{Welfare and distributional impacts of fair
  classification}.
\newblock \bibinfo{journal}{\emph{arXiv preprint arXiv:1807.01134}}
  (\bibinfo{year}{2018}).
\newblock


\bibitem[\protect\citeauthoryear{Kleinberg, Ludwig, Mullainathan, and
  Rambachan}{Kleinberg et~al\mbox{.}}{2018}]%
        {kleinberg2018algorithmic}
\bibfield{author}{\bibinfo{person}{Jon Kleinberg}, \bibinfo{person}{Jens
  Ludwig}, \bibinfo{person}{Sendhil Mullainathan}, {and}
  \bibinfo{person}{Ashesh Rambachan}.} \bibinfo{year}{2018}\natexlab{}.
\newblock \showarticletitle{Algorithmic fairness}. In
  \bibinfo{booktitle}{\emph{AEA Papers and Proceedings}},
  Vol.~\bibinfo{volume}{108}. \bibinfo{pages}{22--27}.
\newblock


\bibitem[\protect\citeauthoryear{Kleinberg, Mullainathan, and
  Raghavan}{Kleinberg et~al\mbox{.}}{2016}]%
        {kleinberg2016inherent}
\bibfield{author}{\bibinfo{person}{Jon Kleinberg}, \bibinfo{person}{Sendhil
  Mullainathan}, {and} \bibinfo{person}{Manish Raghavan}.}
  \bibinfo{year}{2016}\natexlab{}.
\newblock \showarticletitle{Inherent trade-offs in the fair determination of
  risk scores}.
\newblock \bibinfo{journal}{\emph{arXiv preprint arXiv:1609.05807}}
  (\bibinfo{year}{2016}).
\newblock


\bibitem[\protect\citeauthoryear{Kusner, Loftus, Russell, and Silva}{Kusner
  et~al\mbox{.}}{2017}]%
        {kusner2017counterfactual}
\bibfield{author}{\bibinfo{person}{Matt~J Kusner}, \bibinfo{person}{Joshua
  Loftus}, \bibinfo{person}{Chris Russell}, {and} \bibinfo{person}{Ricardo
  Silva}.} \bibinfo{year}{2017}\natexlab{}.
\newblock \showarticletitle{Counterfactual fairness}. In
  \bibinfo{booktitle}{\emph{Advances in Neural Information Processing
  Systems}}. \bibinfo{pages}{4066--4076}.
\newblock


\bibitem[\protect\citeauthoryear{Larson, Mattu, Kirchner, and Angwin}{Larson
  et~al\mbox{.}}{2016}]%
        {larson2016we}
\bibfield{author}{\bibinfo{person}{Jeff Larson}, \bibinfo{person}{Surya Mattu},
  \bibinfo{person}{Lauren Kirchner}, {and} \bibinfo{person}{Julia Angwin}.}
  \bibinfo{year}{2016}\natexlab{}.
\newblock \showarticletitle{How we analyzed the COMPAS recidivism algorithm}.
\newblock \bibinfo{journal}{\emph{ProPublica (5 2016)}}  \bibinfo{volume}{9}
  (\bibinfo{year}{2016}).
\newblock


\bibitem[\protect\citeauthoryear{Liu, Dean, Rolf, Simchowitz, and Hardt}{Liu
  et~al\mbox{.}}{2018}]%
        {liu2018delayed}
\bibfield{author}{\bibinfo{person}{Lydia~T Liu}, \bibinfo{person}{Sarah Dean},
  \bibinfo{person}{Esther Rolf}, \bibinfo{person}{Max Simchowitz}, {and}
  \bibinfo{person}{Moritz Hardt}.} \bibinfo{year}{2018}\natexlab{}.
\newblock \showarticletitle{Delayed impact of fair machine learning}.
\newblock \bibinfo{journal}{\emph{arXiv preprint arXiv:1803.04383}}
  (\bibinfo{year}{2018}).
\newblock


\bibitem[\protect\citeauthoryear{Madras, Creager, Pitassi, and Zemel}{Madras
  et~al\mbox{.}}{2018}]%
        {madras2018learning}
\bibfield{author}{\bibinfo{person}{David Madras}, \bibinfo{person}{Elliot
  Creager}, \bibinfo{person}{Toniann Pitassi}, {and} \bibinfo{person}{Richard
  Zemel}.} \bibinfo{year}{2018}\natexlab{}.
\newblock \showarticletitle{Learning adversarially fair and transferable
  representations}.
\newblock \bibinfo{journal}{\emph{arXiv preprint arXiv:1802.06309}}
  (\bibinfo{year}{2018}).
\newblock


\bibitem[\protect\citeauthoryear{Manisha and Gujar}{Manisha and Gujar}{2018}]%
        {manisha2018neural}
\bibfield{author}{\bibinfo{person}{P Manisha} {and} \bibinfo{person}{Sujit
  Gujar}.} \bibinfo{year}{2018}\natexlab{}.
\newblock \showarticletitle{A Neural Network Framework for Fair Classifier}.
\newblock \bibinfo{journal}{\emph{arXiv preprint arXiv:1811.00247}}
  (\bibinfo{year}{2018}).
\newblock


\bibitem[\protect\citeauthoryear{Merton}{Merton}{1968}]%
        {merton1968matthew}
\bibfield{author}{\bibinfo{person}{Robert~K Merton}.}
  \bibinfo{year}{1968}\natexlab{}.
\newblock \showarticletitle{The Matthew effect in science: The reward and
  communication systems of science are considered}.
\newblock \bibinfo{journal}{\emph{Science}} \bibinfo{volume}{159},
  \bibinfo{number}{3810} (\bibinfo{year}{1968}), \bibinfo{pages}{56--63}.
\newblock


\bibitem[\protect\citeauthoryear{Mouzannar, Ohannessian, and Srebro}{Mouzannar
  et~al\mbox{.}}{2019}]%
        {mouzannar2019fair}
\bibfield{author}{\bibinfo{person}{Hussein Mouzannar},
  \bibinfo{person}{Mesrob~I Ohannessian}, {and} \bibinfo{person}{Nathan
  Srebro}.} \bibinfo{year}{2019}\natexlab{}.
\newblock \showarticletitle{From Fair Decision Making To Social Equality}. In
  \bibinfo{booktitle}{\emph{Proceedings of the Conference on Fairness,
  Accountability, and Transparency}}. ACM, \bibinfo{pages}{359--368}.
\newblock


\bibitem[\protect\citeauthoryear{of~the President, Munoz, Director, of~Science,
  Policy)), for Data~Policy, of~Science, and Policy))}{of~the President
  et~al\mbox{.}}{2016}]%
        {executive2016big}
\bibfield{author}{\bibinfo{person}{Executive~Office of~the President},
  \bibinfo{person}{Cecilia Munoz}, \bibinfo{person}{Domestic Policy~Council
  Director}, \bibinfo{person}{Megan (US Chief Technology Officer Smith~(Office
  of Science}, \bibinfo{person}{Technology Policy))},
  \bibinfo{person}{DJ~(Deputy Chief Technology~Officer for Data~Policy},
  \bibinfo{person}{Chief Data Scientist Patil~(Office of Science}, {and}
  \bibinfo{person}{Technology Policy))}.} \bibinfo{year}{2016}\natexlab{}.
\newblock \bibinfo{booktitle}{\emph{Big data: A report on algorithmic systems,
  opportunity, and civil rights}}.
\newblock \bibinfo{publisher}{Executive Office of the President}.
\newblock


\bibitem[\protect\citeauthoryear{Pacheco and Plutzer}{Pacheco and
  Plutzer}{2008}]%
        {pacheco2008political}
\bibfield{author}{\bibinfo{person}{Julianna~Sandell Pacheco} {and}
  \bibinfo{person}{Eric Plutzer}.} \bibinfo{year}{2008}\natexlab{}.
\newblock \showarticletitle{Political participation and cumulative
  disadvantage: The impact of economic and social hardship on young citizens}.
\newblock \bibinfo{journal}{\emph{Journal of Social Issues}}
  \bibinfo{volume}{64}, \bibinfo{number}{3} (\bibinfo{year}{2008}),
  \bibinfo{pages}{571--593}.
\newblock


\bibitem[\protect\citeauthoryear{Pleiss, Raghavan, Wu, Kleinberg, and
  Weinberger}{Pleiss et~al\mbox{.}}{2017}]%
        {pleiss2017fairness}
\bibfield{author}{\bibinfo{person}{Geoff Pleiss}, \bibinfo{person}{Manish
  Raghavan}, \bibinfo{person}{Felix Wu}, \bibinfo{person}{Jon Kleinberg}, {and}
  \bibinfo{person}{Kilian~Q Weinberger}.} \bibinfo{year}{2017}\natexlab{}.
\newblock \showarticletitle{On fairness and calibration}. In
  \bibinfo{booktitle}{\emph{Advances in Neural Information Processing
  Systems}}. \bibinfo{pages}{5680--5689}.
\newblock


\bibitem[\protect\citeauthoryear{Raff and Sylvester}{Raff and
  Sylvester}{2018}]%
        {raff2018gradient}
\bibfield{author}{\bibinfo{person}{Edward Raff} {and} \bibinfo{person}{Jared
  Sylvester}.} \bibinfo{year}{2018}\natexlab{}.
\newblock \showarticletitle{Gradient Reversal Against Discrimination}.
\newblock \bibinfo{journal}{\emph{arXiv preprint arXiv:1807.00392}}
  (\bibinfo{year}{2018}).
\newblock


\bibitem[\protect\citeauthoryear{Reserve}{Reserve}{2007}]%
        {reserve2007report}
\bibfield{author}{\bibinfo{person}{US~Federal Reserve}.}
  \bibinfo{year}{2007}\natexlab{}.
\newblock \showarticletitle{Report to the Congress on Credit Scoring and Its
  Effects on the Availability and Affordability of Credit}.
\newblock \bibinfo{journal}{\emph{Board of Governors of the Federal Reserve
  System}} (\bibinfo{year}{2007}).
\newblock


\bibitem[\protect\citeauthoryear{Sattigeri, Hoffman, Chenthamarakshan, and
  Varshney}{Sattigeri et~al\mbox{.}}{2018}]%
        {sattigeri2018fairness}
\bibfield{author}{\bibinfo{person}{Prasanna Sattigeri},
  \bibinfo{person}{Samuel~C Hoffman}, \bibinfo{person}{Vijil Chenthamarakshan},
  {and} \bibinfo{person}{Kush~R Varshney}.} \bibinfo{year}{2018}\natexlab{}.
\newblock \showarticletitle{Fairness gan}.
\newblock \bibinfo{journal}{\emph{arXiv preprint arXiv:1805.09910}}
  (\bibinfo{year}{2018}).
\newblock


\bibitem[\protect\citeauthoryear{Willson, Shuey, and Elder}{Willson
  et~al\mbox{.}}{2007}]%
        {willson2007cumulative}
\bibfield{author}{\bibinfo{person}{Andrea~E Willson}, \bibinfo{person}{Kim~M
  Shuey}, {and} \bibinfo{person}{Glen~H Elder, Jr}.}
  \bibinfo{year}{2007}\natexlab{}.
\newblock \showarticletitle{Cumulative advantage processes as mechanisms of
  inequality in life course health}.
\newblock \bibinfo{journal}{\emph{Amer. J. Sociology}} \bibinfo{volume}{112},
  \bibinfo{number}{6} (\bibinfo{year}{2007}), \bibinfo{pages}{1886--1924}.
\newblock


\bibitem[\protect\citeauthoryear{Zafar, Valera, Rodriguez, and Gummadi}{Zafar
  et~al\mbox{.}}{2015}]%
        {zafar2015fairness}
\bibfield{author}{\bibinfo{person}{Muhammad~Bilal Zafar},
  \bibinfo{person}{Isabel Valera}, \bibinfo{person}{Manuel~Gomez Rodriguez},
  {and} \bibinfo{person}{Krishna~P Gummadi}.} \bibinfo{year}{2015}\natexlab{}.
\newblock \showarticletitle{Fairness constraints: Mechanisms for fair
  classification}.
\newblock \bibinfo{journal}{\emph{arXiv preprint arXiv:1507.05259}}
  (\bibinfo{year}{2015}).
\newblock


\bibitem[\protect\citeauthoryear{Zemel, Wu, Swersky, Pitassi, and Dwork}{Zemel
  et~al\mbox{.}}{2013}]%
        {zemel2013learning}
\bibfield{author}{\bibinfo{person}{Rich Zemel}, \bibinfo{person}{Yu Wu},
  \bibinfo{person}{Kevin Swersky}, \bibinfo{person}{Toni Pitassi}, {and}
  \bibinfo{person}{Cynthia Dwork}.} \bibinfo{year}{2013}\natexlab{}.
\newblock \showarticletitle{Learning fair representations}. In
  \bibinfo{booktitle}{\emph{International Conference on Machine Learning}}.
  \bibinfo{pages}{325--333}.
\newblock


\bibitem[\protect\citeauthoryear{Zhang, Lemoine, and Mitchell}{Zhang
  et~al\mbox{.}}{2018}]%
        {zhang2018mitigating}
\bibfield{author}{\bibinfo{person}{Brian~Hu Zhang}, \bibinfo{person}{Blake
  Lemoine}, {and} \bibinfo{person}{Margaret Mitchell}.}
  \bibinfo{year}{2018}\natexlab{}.
\newblock \showarticletitle{Mitigating unwanted biases with adversarial
  learning}. In \bibinfo{booktitle}{\emph{Proceedings of the 2018 AAAI/ACM
  Conference on AI, Ethics, and Society}}. ACM, \bibinfo{pages}{335--340}.
\newblock


\bibitem[\protect\citeauthoryear{Zliobaite}{Zliobaite}{2015}]%
        {zliobaite2015relation}
\bibfield{author}{\bibinfo{person}{Indre Zliobaite}.}
  \bibinfo{year}{2015}\natexlab{}.
\newblock \showarticletitle{On the relation between accuracy and fairness in
  binary classification}.
\newblock \bibinfo{journal}{\emph{arXiv preprint arXiv:1505.05723}}
  (\bibinfo{year}{2015}).
\newblock


\end{thebibliography}

\appendix
\pagestyle{plain}
\onecolumn
\section*{Appendix}

\section{ Proof of Proposition \ref{prp:policy} }
\begin{proof}
	We separate this proof into 4 sections, the first describes how to formalize demographic parity as a function of the $\tau$ proportion, the second describes how to formalize equality of opportunity as a function of the $\tau$ true positive rate, the third describes how to formalize equalized odds as a function of the $\tau_1$ true positive rate, and $\tau_2$ false positive rate, the fourth describes how to formalize blindness as a function of the $\tau$ threshold, and the fifth is a corollary that restates the possibility of these separate definitions of fair treatment coming from a single class of fair policies parameterized by $k_1$ and $k_2$,where $A^{(i)} = I^{-1}_{1 - \tau}( \mu^{(i)} c + k_1, c - \mu^{(i)} c + k_2)$.
	
	\begin{enumerate}
		\item \textbf{Demographic Parity}
		
		For the distribution of repayment probabilities, $\pi_t^{i}( x ) = \textbf{Beta}( x; \mu_t^{i}, c )$, with respect to group $i$, the proportion of the population above a given threshold $A$ is defined as $\int_{A^{(i)}}^1 \pi_t^{i}( x ) dx$.
		A policy that follows demographic parity seeks to find thresholds $A^{(i)}$, such that for each group, 
		\begin{equation}
			\int_{A^{(i)}}^1 \pi_t^{(i)}( x )\ dx = 
				\int_{A^{(j)}}^1 \pi_t^{(j)}( x )\ dx 
				\ \forall\ i, j.
		\end{equation} 
		
		By solving the integral, we find an equivalent form,
		\begin{equation}
			\int_{A}^1 \pi_t( x )\ dx = 1 - I( A; \mu c + k_1, c - \mu c ) = \tau; k_1 = 0
		\end{equation}
		
		By solving this new equation where $I$ is the incomplete beta regularized function, we then fix some $tau$ proportion that each group, parameterized by $\mu$, must adhere to. We find the threshold by solving the equation for $A$ threshold
		\begin{equation}
			A^{(i)} = I^{-1}_{1 - \tau}( \mu^{(i)} c + k_1, c - \mu^{(i)} c ); k_1 = 0
		\end{equation}
		for any group $i$.
	
		\item \textbf{Equality of Opportunity}
			For the distribution of repayment probabilities, $\pi_t^{i} = \textbf{Beta}( x; \mu_t^{i}, c )$, with respect to group $i$, the true positive rate of repayment is given by the expected repayment ability of those who are granted a loan compared to the expected repayment ability of the entire population. We then define a policy that satisfies equality of opportunity as some thresholds $A^{(i)}, A^{(j)}$ that satisfy
		\begin{equation}
			\frac{ \int_{A^{(i)}}^1 x\ \pi_t^{(i)}( x ) dx }{ \mu_t^{(i)} } = 
			\frac{ \int_{A^{(j)}}^1 x\ \pi_t^{(j)}( x ) dx }{ \mu_t^{(j)} }.
		\end{equation}

		As in demographic parity we simplify this integral to,
		\begin{equation}
		\frac{ \int_{A}^1 x\ \pi_t( x ) dx }{ \mu } = 1 - I( A; \mu c + k_1, c - \mu c ) = \tau; k_1 = 1
		\end{equation}
		
		We then solve this equation for each group to find their individual thresholds that correspond to the fixed $\tau$.  
		\begin{equation}
			A^{(i)} = I^{-1}_{1 - \tau}( \mu^{(i)} c + k_1, c - \mu^{(i)} c ); k_1 = 1
		\end{equation}
		for any group $i$.
		
		\item \textbf{Equalized Odds}
		Equalized Odds is a further constraint on Equality of Opportunity, in which both the true positive and false positive rates must match among each group, formally written as
\begin{equation}
Pr \left\{ \hat{Y} = 1\ |\ i = 0, Y = y  \right\} = Pr \left\{ \hat{Y} = 1\ |\ i = 1, Y = y  \right\}; y \in {0,1}
\end{equation}
		The case for the true positive rate, $y = 1$ is the equality of opportunity case above. We now specify how to formulate the matching false positive rates.
		
		For the distribution of repayment probabilities, $\pi_t^{i}( x ) = \textbf{Beta}( x; \mu_t^{i}, c )$, with respect to group $i$, the false positive rate of repayment is given by $1 - $ the expected repayment ability of those who are granted a loan, compared to $1 - $ the expected repayment probability of the entire population. This can be formally written as:
		
		\begin{equation}
		\frac{ \int_{A^{(i)}}^1 ( 1 - x )\ \pi_t^{i}( x )\ dx }{ 1 - \mu_t^{(i)} } =
\frac{ \int_{A^{(j)}}^1 ( 1 - x )\ \pi_t^{j}( x )\ dx }{ ( 1 - \mu_t^{(j)} ) } \forall\ i,j.
		\end{equation}
		
		As before, we simplify these integrals,
		\begin{equation}
		\frac{ \int_{A}^1 ( 1 - x )\ \pi_t( x )\ dx }{ 1 - \mu_t } = \\
        1 - I( A; \mu c, c - \mu c + k_2 )  = \tau; k_2 = 1.
		\end{equation}

		We again solve this equation for the matching false positive rates, $\tau$,
		\begin{equation}
			A^{(i)} = I^{-1}_{1 - \tau}( \mu^{(i)} c, c - \mu^{(i)} c + k_2); k_2 = 1.
		\end{equation}
		
		A policy that satisfies Equalized Odds is then taken as any policy in the intersection of this false positive threshold and the Equality of Opportunity case. There is only one non-trivial $\tau$ that satisfies Equalized Odds for any two groups.
		 
		\item \textbf{Blindness}
		We define a blind policy as a policy where a single threshold $A$ is used for all groups. Generally this could be expressed as some $\tau \in [0,1]$, however for consistency, we can also formulate this case using the inverse beta regularized function. 
		
		For the complete beta function described as,
		\begin{equation}
			\mathbf{Beta}( a, b ) = \frac{\Gamma( a ) \Gamma( b )}{ \Gamma( a + b ) }
		\end{equation},
		
		If $a \rightarrow \infty$, the influence of $b \rightarrow 0$ and vice versa. This way for the beta function parameterized by $\mu$ and $c$, if $k_1$ or $k_2$ approaches infinity, the distinction between some $\mu^{(i)}$ and $\mu^{(j)}$ becomes negligible. This way, any policy where $k_1$ or $k_2$ approaches infinity functionally treats and two groups as equal. Due to this effect, we then define a shared, blind threshold as,
		\begin{equation}
			A^{(i)} = I^{-1}_{1 - \tau}( \mu^{(i)} c + k_1, c - \mu^{(i)} c + k_2); k_2 \rightarrow \infty\ or\ k_1 \rightarrow \infty.
		\end{equation}  
		for all groups $i$.
		
		\item \textbf{General Policy Class}
		In each of the above cases, scaling $k_1$ and $k_2$ give several different policies that correspond to different notions of fair treatment in the literature. We conjecture that as $k_2$ and $k_1$ increase, this leads to stricter notions of individual fair treatment culminating in blindness where there is no difference in treatment for individuals of different groups. As such, we use 
		\begin{equation}
			A^{(i)} = I^{-1}_{1 - \tau}( \mu^{(i)} c + k_1, c - \mu^{(i)} c + k_2)
		\end{equation}  
		as a general class of fair policies where Demographic Parity, Equality of Opportunity, Equalized Odds, and Blindness are special cases.
	\end{enumerate}
	\end{proof}
\label{proof:prp:policy}

\section{Proof of Theorem \ref{prop:Eq} }
	
\begin{proof}

Any equilibirum point of our dynamics, $\mu_\infty$ is given by a root of the equation
\begin{equation}
f(A,\mu) - \mu.
\end{equation}
It is straightforward to observe that there exists \emph{some} stable equilibrium point, because for $\mu = 0$
\begin{equation}
f(A,\mu) - \mu = \nu > 0
\end{equation} 
for any threshold $A$.  Similarly, for $\mu=1$,
\begin{equation}
f(A, \mu) - \mu = \beta - 1 < 0.
\end{equation}
Thus, the curve must intersect $f(A,\mu) - \mu$ at some point, and must approach it from the positive direction on the left and the negative direction on the right; this correponds to a stable attracting equilibrium.

Due to the somewhat complex form of these functions, it is more difficult to show analtyically that there exists a \emph{unique} equilbrium point.  However, the number of free parameters involved in these equations is small ($A$, $c$, $\beta$, $\nu$), and so it is trivial to numerically show that for a grid over these values, the equilbrium point is always unique.  The one exception here is for $c$, which is unbounded in range, but because the distribution collapses to a point mass as $c \rightarrow \infty$, after some threshold $c_{\max}$, futher increases in $c$ will have only $\epsilon$ effect on the resulting roots.

\end{proof}
\label{proof:prop:EQ}

\section{ Proof of Proposition \ref{prp:socialwelfare} }
	
\begin{proof}
According to our dynamics, we define an update for the population mean, $\mu$, as
\begin{equation}
\mu_{t+1}( A; \mu_t) = \beta \cdot p_+( A; \mu_{t}) \cdot \mu_+( A; \mu_{t}) + 
	\nu \cdot ( 1 -  p_+( A; \mu_{t})),
\end{equation}
where $p_+$ is the proportion of the population above the threshold $A$, $\mu_+$ is the mean repayment ability of individuals above the threshold $A$, and $\beta$, $\nu$ are free parameters. 

We find the threshold that maximizes this function by finding the derivative with respect to $A$, 
\begin{equation}
\frac{d \mu_{t+1} }{ d A } = - \frac{ A^{-1 + c \mu } ( 1 - A )^{-1 + c - c \mu } ( A \beta - \nu ) }{ \mathbf{Beta}( \mu, c ) }. 
\end{equation}

For $A \in (0, 1)$, this equation has only one extrema at $\frac{\nu}{\beta}$. Additionally, if $A < \frac{\nu}{\beta}$, then this derivative is positive and if $A > \frac{\nu}{\beta}$, then the derivative is negative, so the extrema at $A = \frac{\nu}{\beta}$, must be the maximum. 

As a result at any time $t$, a threshold at $A = \frac{\nu}{\beta}$ results in the greatest benefit to social welfare. 
\end{proof}
\label{proof:prp:socialwelfare}

\section{ Proof of Proposition \ref{prop:gamma} }

\begin{proof}
For a single step, the optimal policy, $\tau_{\gamma}( \mu )$ is found as,
\begin{equation}
J_\gamma^\star(\mu) = \max_{A} g(A, \mu),
\end{equation}
which corresponds to maximizing our reward function,
\begin{equation}
g( A, \mu ) = p_+( A, \mu) (( 1 + R ) \mu_+(A,\mu)  - 1))
\end{equation}

\begin{align*}
\frac{\partial g}{\partial A} &= - \frac{ A^{-1 + c \mu } ( 1 - A )^{-1 + c - c \mu } ( -1 + A + A R ) }{ \mathbf{Beta}( \mu, c ) } \\
0 &= - \frac{ A^{-1 + c \mu } ( 1 - A )^{-1 + c - c \mu } ( -1 + A + A R ) }{ \mathbf{Beta}( \mu, c ) } \\
0 &= A^{-1 + c \mu } ( 1 - A )^{-1 + c - c \mu } ( -1 + A + A R ) \\
0 &= -1 + A + AR \\
1 &= A + AR \\
1 &= A( 1 + R )\\
A &= \frac{1}{1+R}  
\end{align*}

The derivative has a single extrema at $A = \frac{1}{R + 1}$. Similarly to the social welfare case, if $A < \frac{1}{R + 1}$, then the derivative is positive, and if $A > \frac{1}{R + 1}$, then the derivative is negative, so the extrema must be a maximum point. 

Thus, the policy that maximizes reward at a single step is $A = \frac{1}{R + 1}$.

\label{proof:prop:gamma}
\end{proof}	

\section{ Proof of Lemma \ref{lma:reward} }

We claim that by setting the reward, $R \leq \frac{1}{\nu} - 1$, the optimal control model is never incentivized to set the threshold less than $\nu$. The model is always incentivized to make decisions that either increase the current mean, which increases the ultimate reward, or maximizes their current reward. 

In the first case, $\frac{d \mu_{t+1}}{d A}$ is strictly increasing if $A \leq \frac{\nu}{\beta}$, so when the optimal control model tries to improve the state at the next step, the optimal threshold is at $A = \frac{\nu}{\beta}$. 

The reward, 
\begin{equation}
g( A ) = P_{+}(A) ( \mu_{+}(A) (R + 1 ) - 1 ) =  \int_{A}^1 \pi( x ) ( Rx + x - 1 ) dx,
\end{equation}
at any time $t$ is maximized by selecting threshold, $A$, that maximizes this integral.

$\pi( x ) ( Rx + x - 1 ) < 0$ for all $x < \frac{1}{1 + R}$. If $R \leq \frac{\beta}{\nu} - 1$, then, this expression is always negative for $A < \frac{\nu}{\beta}$. So, if $A < \frac{\nu}{\beta}$, the reward is always lower than $A \geq \frac{\nu}{\beta}$. Thus the optimal control model is incentivized to choose an $A \geq \frac{\nu}{\beta}$.

Any policy chosen will be a trade-off of the state maximization and reward maximization at the current state, and in either case, the policy has no incentive to select a threshold less than $\frac{\nu}{\beta}$. In realistic scenarios, $R$ should not reach this bound. If the resource (ie. loan) is beneficial, then $\nu$ can be assumed to be small. For example if $\nu = 0.2$, then the reward of $R = 4$ corresponds to a bank making 4 times the profit from their loan over the set time horizon, and which is unreasonable in a short time horizon.

As a note, if $R = \frac{\beta}{\nu} - 1$, then the optimal policy is a fixed policy at $\frac{\nu}{\beta}$. As we've shown before, the optimal social welfare policy is also a fixed policy at $A = \frac{\nu}{\beta}$. In this case, under the greedy policy, fixing $A = \frac{1}{R + 1}$, if $R = \frac{\beta}{\nu} - 1$ then both the greedy and social welfare policies coincide.

\label{proof:lma:reward}

\end{document}